\documentclass[nohyperref]{article}

\usepackage{microtype}
\usepackage{graphicx}
\usepackage{subfigure}
\usepackage{booktabs} % for professional tables

\usepackage{hyperref}

\usepackage[accepted]{icml2022}

\usepackage{amsmath}
\usepackage{amssymb}
\usepackage{mathtools}
\usepackage{amsthm}

\usepackage[capitalize,noabbrev]{cleveref}

\theoremstyle{plain}
\newtheorem{theorem}{Theorem}[section]
\newtheorem{proposition}[theorem]{Proposition}
\newtheorem{lemma}[theorem]{Lemma}

\theoremstyle{definition}

\theoremstyle{remark}

\usepackage[textsize=tiny]{todonotes}

\icmltitlerunning{Channel Importance Matters in Few-Shot Image Classification}

\usepackage{url}
\usepackage{bm}

\def \xx {\bm{x}}
\def \zz {\bm{z}}

\def \hz {\bm{\widetilde{z}}}
\def \uu {\bm{\mu}}
\def \hu {\bm{\widetilde{\mu}}}
\def \oo {\bm{\omega}}

\def \SS {\bm{\Sigma}}

\def \ss {\bm{\sigma}}

\def \wz {\widetilde{z}}
\def \wu {\widetilde{\mu}}
\def \ws {\widetilde{\sigma}}

\usepackage{multirow}
\newlength{\Oldarrayrulewidth}
\newcommand{\acc}[2]{#1\scriptsize\textcolor[RGB]{168,8,13}{+#2}}
\newcommand{\macc}[2]{#1\scriptsize\textcolor[RGB]{61,145,64}{-#2}}
\newcommand{\zacc}[2]{#1\scriptsize\textcolor[RGB]{61,145,64}{+#2}}
\newcommand{\pacc}[2]{#1\scriptsize\textcolor[RGB]{168,8,13}{$\pm$#2}}

\begin{document}

\twocolumn[
% \icmltitle{Unravelling the Secret of Few-Shot Learning\\ through a Simple Feature Transformation}

% \icmltitle{Understanding the Difficulty of Representation Transfer in Few-Shot Learning}

% \icmltitle{Rethinking Representation Transfer in Few-Shot Classification\\ with a Simple Feature Transformation}
% \icmltitle{Rethinking the Difficulty of Representation Transfer in Few-Shot Classification}
\icmltitle{Channel Importance Matters in Few-Shot Image Classification}

% \icmltitle{Rethinking Few-Shot Representation Transferability of Vision Models\\with a Simple Channel-wise Feature Transformation}

% \icmltitle{Personalized Federated Learning with Theoretical Guarantees: A Model-Agnostic Meta-Learning Approach}



\icmlsetsymbol{equal}{*}

\begin{icmlauthorlist}
\icmlauthor{Xu Luo}{1}
\icmlauthor{Jing Xu}{2}
\icmlauthor{Zenglin Xu}{2,3}
% \icmlauthor{Firstname3 Lastname3}{comp}
\end{icmlauthorlist}

\icmlaffiliation{1}{University of Electronic Science and Technology of China}
\icmlaffiliation{2}{Harbin Institute of Technology Shenzhen}
\icmlaffiliation{3}{Pengcheng Laboratory}

\icmlcorrespondingauthor{Xu Luo}{frank.luox@outlook.com}
\icmlcorrespondingauthor{Zenglin Xu}{xuzenglin@hit.edu.cn}

\icmlkeywords{Machine Learning, ICML}

\vskip 0.3in
]



\printAffiliationsAndNotice{}  % leave blank if no need to mention equal contribution
\begin{abstract}
Few-Shot Learning (FSL) requires vision models to quickly adapt to brand-new classification tasks with a shift in task distribution. Understanding the difficulties posed by this \emph{task distribution shift} is central to FSL. In this paper, we show that a simple channel-wise feature transformation may be the key to unraveling this secret from a channel perspective. When facing novel few-shot tasks in the test-time datasets, this transformation can greatly improve the generalization ability of learned image representations, while being agnostic to the choice of datasets and training algorithms. Through an in-depth analysis of this transformation, we find that the difficulty of representation transfer in FSL stems from the severe \emph{channel bias} problem of image representations: channels may have different importance in different tasks, while  convolutional neural networks are likely to be insensitive, or respond incorrectly to such a shift. This  points out a core problem of the generalization ability of modern vision systems which needs further attention in the future. Our code is available at \url{https://github.com/Frankluox/Channel_Importance_FSL}.

% Some potential solutions, except from this simple transformation, are not able to completely solve this problem.
% the emphasized channels almost keep the same between training and testing
% Some other potential solutions, such as direct estimation of  unsupervised contrastive leaning, fine-tuning and feature 

% This transformation is able to improve FSL performance across a wide range choices of algorithms, training datasets and test datasets, but .

% We show that a simple transformation applied on the top of pre-trained neural networks can largely improve FSL performance across a wide range choices of algorithms, training datasets and test datasets. The wide applicability of this transformation leads us to utilize it to unravel

% We then 

% use it as a tool to characterize the core problems in FSL

% show that a frustratingly simple transformation applied on the top of pre-trained neural networks can reliably improve the Few-Shot transfer performance only when novel classes are introduced. This task-specific improvement makes the transformation an ideal tool to characterize the core problems in FSL. 

\end{abstract}
\section{Introduction}
Deep convolutional neural networks~\cite{Alexnet,Resnet} have revolutionized computer vision in the last decade, making it possible to automatically learn representations from a large number of images. The learned representations can generalize well to brand-new images. As a result,  image  classification performance is close to humans on most benchmarks. However, in addition to recognizing previously-seen categories, humans can quickly change their focus of image patterns in changing environments and recognize new categories given only a few observations. This fast learning capability, known as Few-Shot Learning (FSL), challenges current vision models on the ability to quickly adapt to novel classification tasks that are different from those in training. This \emph{task distribution shift} means that categories, domains of images or granularity of categories in new tasks deviate from those in the training tasks.

% This \emph{task distribution shift} may include category, domain or granularity shift.
% These test-time tasks are sampled from a distribution different from training,   with \emph{task distribution shift}. which means that categories, domain or granularity of the task deviate from those of training tasks. 

Recent studies of few-shot image classification have highlighted the importance of the quality of learned image representations~\cite{mamlrepresentation,crosstransformer,baseline,allyouneed,IER}, and also showed that representations learned by neural networks do not generalize well to novel few-shot classification tasks when there is task distribution shift~\cite{metabaseline,crosstransformer,sensitivity}. Thus it is crucial to understand how task distribution shift affects the generalization ability of image representations in few-shot learning.

\begin{figure}[t]
% \vskip 0.2in
\centering
\centerline{\includegraphics[width=1.0\columnwidth]{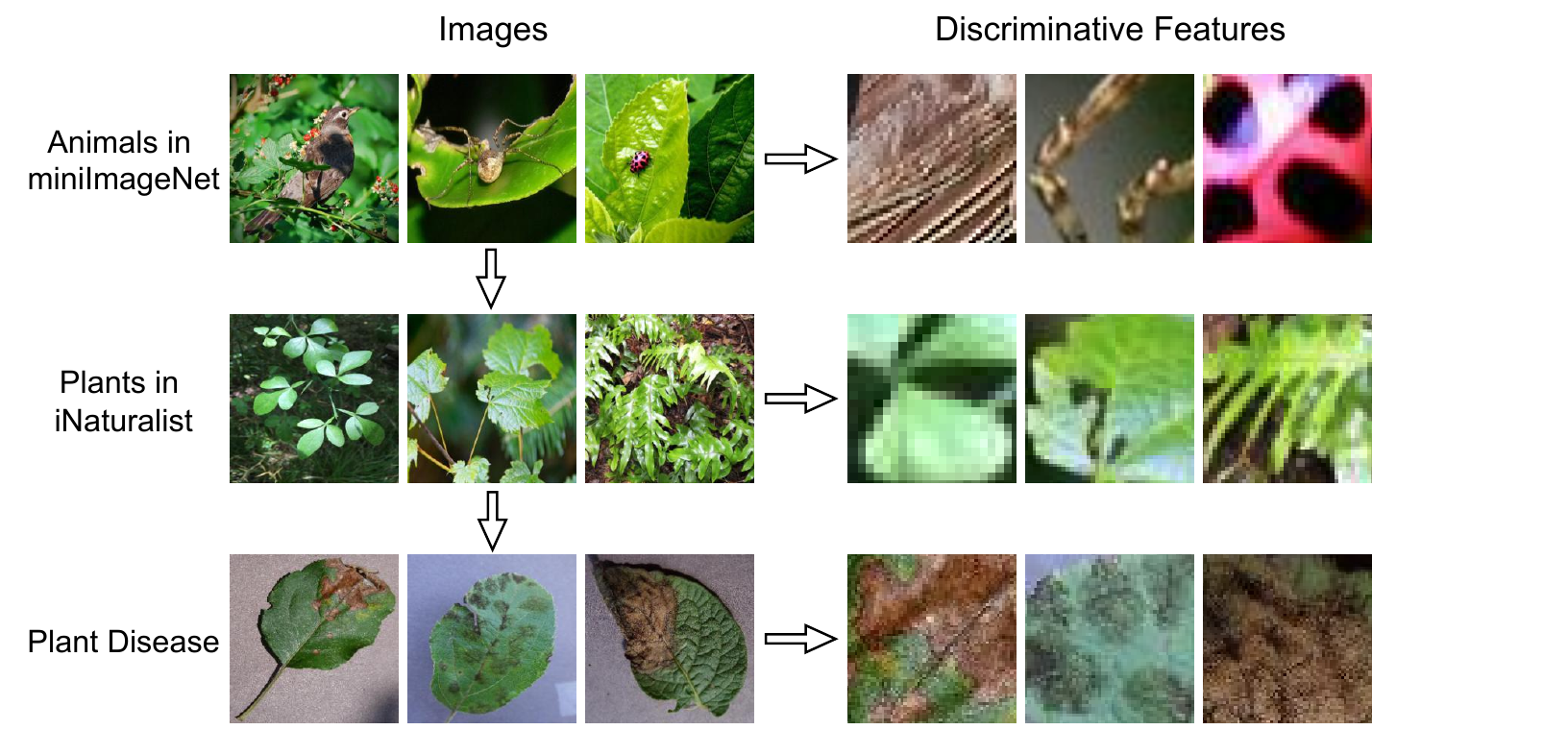}}
\caption{\textbf{Examples of task distribution shift.} Different classification tasks may focus on distinct discriminative information. Top: animals in \emph{mini}ImageNet with different plants as background. Middle: plants as the main categories in iNaturalist. Bottom: Different types of plant diseases in the fine-grained Plant Disease dataset.}
% Although a model trained on \emph{mini}ImageNet has seen various plant data as background and encode some relevant information, it may not give it enough focus, thereby performing bad on tasks where plants are the main categories. Even when the types of main objects are the same, the granularity of tasks may also influence the discriminative information (e.g. trained on iNaturalist, tested on Plant Disease).}
\label{example_channelchange}

\vskip -0.25in
\end{figure}

% assign some channels of its representation to encode relevant information
As shown in Figure \ref{example_channelchange}, task distribution shift may lead to changes in discriminative image features that are critical to the classification task at hand. For example, in the task of recognizing animals, a convolutional neural network trained on \emph{mini}ImageNet can successfully identify the discriminative information related to animals. Although the representations learned by the network may encode some plant information (from image background), plants do not appear as a main category in miniImageNet and it may be insufficient for the network to distinguish various plants in a novel few-shot task sampled from the iNaturalist dataset. Even when the network is well trained to recognize plants on iNaturalist, it is difficult to be adapted to the novel task of identifying plant diseases due to the granularity shift, since the discriminative information now becomes the more fine-grained lesion part of leaves.

In this paper, we show that this difficulty encountered in few-shot learning leads to a \emph{channel bias} problem in learned image representations (i.e., features). Specifically, in the layer after global pooling, different channels in the learned feature seek for different patterns~(as verified in~\cite{emergingobject,dissection}) during training, and the channels are weighted (in a biased way) based on their importance to the training task.  However, when applied to novel few-shot classification tasks, the learned image features usually do not change much or have inappropriately changed without adapting to categories in novel tasks. This bias towards training tasks may result in imprecise attention to image features in novel tasks.

% it can be shown that the representations in 

% current few-shot learning methods may not capture 

% This leads to a significant bias of image channels.

What leads to our discovery of the channel bias problem is a simple transformation function that we found in a mathematical analysis textbook. Applied to top of image representations channel-wisely only at test time on the fly, this transformation function can consistently and largely improve predictions for out-of-distribution few-shot classification tasks, being agnostic to the choice of datasets and training algorithms~(e.g., 0.5-7.5\% average improvement over 5-way 5-shot tasks on 19 different test-time datasets, as shown in Table~\ref{performance}). Through analysis, we reveal the existence of channel bias problem, and show that this transformation rectifies channel emphasis by adjusting the Mean Magnitude of Channels (MMC) of image representations over the target 
task. Concretely, it serves as a smoothing function that suppresses channels of large MMC and largely amplifies channels of small MMC. 

% The channel bias problem can be alleviated by a simple channel-wise feature transformation function presented in this work~(see Figure \ref{transformation_plot}). This transformation function is applied to image representations channel-wisely only at test-time on the fly. It changes channel emphasis by adjusting the Mean Magnitude of Channels (MMC) of image representations over the target classes. Concretely, it serves as a smoothing function that suppresses channels of large MMC and amplifies channels of small MMC. Empirical studies show that the transformation function
% can consistently and largely improve predictions for out-of-distribution few-shot classification tasks, being agnostic to the choice of training algorithms and datasets~(e.g., 0.5-7.5\% average improvement over 5-way 5-shot tasks on 19 different test-time datasets, as shown in Table~\ref{performance}). 

% lift up importance of those channels that matter for test-time tasks.

% suppresses channels of large magnitude and amplifies channels of small magnitude. 

% We give evidence that this transformation works by adjusting the Mean Magnitude of Channels (MMC) of image representations on the target tasks, which we show can be treated as a measurement of how much emphasis neural networks put on each channel.
% This should blame, at least partly, for the performance drop on novel tasks. 

% where the emphasis on channels of representations should be greatly adjusted, we show that

%1. 太多we 2. 不要说should be greatly adjusted

% We measure this channel emphasis by the Mean Magnitude of Channels (MMC).

To further understand the channel bias problem, we derive an oracle adjustment on the MMC of image representations in binary classification tasks. Such studies demonstrate that the channel bias problem exists in many different target tasks with various types of task distributions shift, and it becomes severe with the distribution shift expanding (as shown in Figure~\ref{visiualize_optimal}). In addition, through test-time shot analysis, we verify that the channel bias problem requires more attention in few-shot setting, while simple fine-tuning can help address this problem in many-shot setting.

\section{A Channel-wise Feature Transformation}
\subsection{Problem Setup}
In few-shot image classification, a training set $\mathcal{D}^{train}$ is used at first to train a neural network parametrized by $\theta$, which will be evaluated on a series of few-shot classification tasks constructed from the test-time dataset $\mathcal{D}^{test}$. Importantly,
there should be task distribution shift between $\mathcal{D}^{train}$ and $\mathcal{D}^{test}$, which may include category shift, domain shift or granularity shift. Each evaluated $N$-way $K$-shot few-shot classification task $\tau$ is constructed by first sampling $N$ classes from $\mathcal{D}^{test}$, and then sampling $K$ and $M$ images from each class to constitute a support set $\mathcal{S}_\tau$ and a query set $\mathcal{Q}_\tau$, respectively. The support set $\mathcal{S}_\tau=\{(x_{k,n}^\tau, y_{k,n}^\tau)\}_{k,n=1}^{K,N}$ consisting of $K\times N$ images $x_{k,n}^\tau$ and corresponding labels $y_{k,n}^\tau$ from the $N$ classes is used to construct a classifier $p_\theta(\cdot|x, \mathcal{S}_\tau)$, which is further evaluated on the query set $\mathcal{Q}_\tau=\{x_{m,n}^{*\tau}\}_{m,n=1}^{M,N}$. The evaluation metric is the average prediction accuracy on query set over all sampled few-shot classification tasks.

\begin{table*}[t]
\setlength\tabcolsep{2.5pt}
\footnotesize
% \scriptsize
\caption{\textbf{Performance gains of the simple feature transformation on various training and testing datasets with a broad range of choices of network architectures and algorithms}. The black values indicate the original accuracy, and the red values indicate the increase.  Each running of evaluation contains 10000 5-way 5-shot tasks sampled using a fixed seed, and the average accuracy is reported. The three groups of test-time datasets come from MetaDataset, BSCD-FSL benchmark and DomainNet, respectively.}
\label{performance}
\centering
\begin{tabular}{c|ccccccc|ccc|c|c}
\\[-1em]
TrainData &  \multicolumn{7}{|c|}{\emph{mini}-train} & \multicolumn{3}{|c|}{ImageNet} & iNaturalist\\
Algorithm & PN & PN & CE & MetaB & MetaOpt & CE & S2M2 & PN & CE & MoCo-v2 & CE\\
Architecture & Conv-4 & Res-12 & Res-12 & Res-12 & Res-12 & SE-Res50 & WRN & Res-50 & Res-50 & Res-50 & Res-50 & Average\\ \hline
\emph{mini}-test & \acc{66.6}{1.2} & \acc{73.5}{2.2} & \acc{75.9}{1.6} & \acc{74.7}{2.6} & \acc{74.8}{0.5} & \acc{76.2}{0.2} & \acc{82.5}{1.2} & \macc{82.2}{1.6} & \macc{89.1}{0.5} & \acc{93.7}{2.2} & \acc{69.9}{2.2} & \acc{78.1}{1.1}\\
CUB & \acc{52.0}{2.8} & \acc{57.0}{3.0} & \acc{59.6}{2.3} & \acc{60.1}{2.6} & \acc{60.3}{1.7} & \acc{59.9}{2.2} & \acc{68.5}{2.8} & \acc{65.3}{2.5} & \acc{78.2}{0.4} & \acc{70.0}{6.8} & \zacc{94.7}{0.0} & \acc{66.0}{2.5}\\
Textures & \acc{50.9}{2.3} & \acc{57.1}{4.2} & \acc{63.1}{2.4} & \acc{61.2}{3.7} & \acc{60.2}{1.8} & 
\acc{63.5}{0.6} & \acc{69.3}{2.9} & \acc{61.9}{2.4} & \acc{71.6}{0.8} & \acc{82.8}{0.9} & \acc{63.2}{2.3} & \acc{64.1}{2.0}\\
Traffic Signs & \acc{52.6}{2.1} & \acc{64.8}{2.2} & \acc{65.6}{1.4} & \acc{67.3}{1.5} & \acc{67.1}{4.9} & 
\acc{62.2}{2.9} & \acc{69.6}{3.1} & \acc{64.0}{2.2} & \acc{67.2}{3.5} & \acc{68.4}{8.8} & \acc{60.5}{4.0} &\acc{64.4}{3.3}\\
Aircraft & \acc{32.1}{0.9} & \acc{31.3}{1.6} & \acc{34.7}{1.9} & \acc{34.7}{2.3} & \acc{35.6}{2.4} & \acc{38.2}{2.0} & \acc{40.5}{4.7} & \acc{38.4}{1.7} & \acc{46.6}{2.5} & \acc{34.5}{8.8} & \acc{42.1}{2.5} & \acc{34.0}{2.9}\\

Omniglot & \acc{61.0}{10.0} & \acc{77.6}{7.8} & \acc{86.9}{3.7}& \acc{81.6}{7.9} & \acc{78.0}{9.9} & \acc{89.9}{2.3} & \acc{85.9}{7.4} & \acc{76.4}{2.9} & \acc{88.6}{5.3} & \acc{74.5}{15.8} & \acc{83.8}{9.0} & \acc{80.4}{7.5}\\

VGG Flower & \acc{71.0}{3.1} & \acc{71.1}{5.5} & \acc{79.2}{3.8} & \acc{78.3}{4.5} & \acc{78.4}{3.1} & \acc{83.0}{1.7} & \acc{87.8}{2.5} & \acc{81.4}{2.6} & \acc{89.3}{1.7} & \acc{86.2}{6.3} & \acc{91.9}{1.1} & \acc{81.6}{3.3}\\

MSCOCO & \acc{52.0}{1.2} & \acc{58.2}{1.1} & \acc{59.0}{0.7} & \acc{58.0}{1.6} & \acc{58.4}{0.1} & \acc{57.1}{0.5} & \acc{63.5}{0.1} & \macc{61.3}{0.5} & \macc{64.3}{0.4} & \acc{71.4}{1.4} & \acc{50.4}{1.9} & \acc{59.4}{0.7}\\

Quick Draw & \acc{49.7}{6.5} & \acc{60.2}{5.4} & \acc{67.5}{6.5} & \acc{61.9}{9.0} & \acc{61.0}{6.2} & \acc{69.8}{2.8} & \acc{66.4}{8.2} & \acc{59.8}{6.9} & \acc{70.2}{3.0} & \acc{63.7}{8.3} & \acc{60.8}{6.2} & \acc{62.8}{6.3}\\
Fungi & \acc{48.5}{1.5} & \acc{49.0}{3.7} & \acc{52.2}{3.3} & \acc{51.5}{4.0} & \acc{54.6}{1.9} & \acc{55.2}{0.5} & \acc{61.6}{3.8} & \acc{58.5}{1.3} & \acc{65.1}{1.1} & \acc{60.2}{9.2} & \acc{70.0}{1.8} & \acc{56.9}{2.9}
\\ \hline
Plant Disease & \acc{66.6}{7.8} & \acc{73.3}{7.9} & \acc{80.0}{5.1} & \acc{75.6}{7.6} & \acc{78.6}{4.5} & \acc{83.1}{3.2} & \acc{86.4}{3.5} & \acc{72.5}{8.0} & \acc{84.1}{3.3} & \acc{87.1}{4.7} & \acc{85.6}{4.1} & \acc{79.4}{5.4}\\
ISIC & \acc{38.5}{1.6} & \acc{36.8}{2.9} & \acc{40.4}{1.0} & \acc{38.8}{1.7} & \acc{39.5}{2.3} & \acc{37.7}{3.9} & \acc{40.5}{5.5} & \acc{39.5}{4.0} & \acc{37.8}{3.6} & \acc{43.2}{2.8} & \acc{39.0}{4.3} & \acc{39.2}{3.1}\\
EuroSAT & \acc{63.0}{4.5} & \acc{67.3}{5.5} & \acc{75.7}{2.9} & \acc{71.9}{4.5} & \acc{72.8}{5.8} & \acc{75.7}{1.6} & \acc{81.2}{2.9} & \acc{72.5}{6.1} & \acc{78.4}{2.2} & \acc{83.5}{2.7} & \acc{73.5}{3.7} & \acc{74.1}{3.9}\\
ChestX & \acc{22.9}{0.2} & \acc{23.0}{0.5} & \acc{24.1}{0.3} & \acc{23.5}{0.5} & \acc{24.5}{0.4} & \acc{23.6}{0.2} & \acc{24.2}{0.9} & \acc{23.2}{0.3} & \acc{24.2}{0.8} & \acc{25.4}{0.9} & \acc{23.9}{0.1} & \acc{23.9}{0.5}\\ \hline

Real & \acc{67.0}{1.8} & \acc{72.2}{3.1} & \acc{76.3}{1.6} & \acc{75.0}{2.6} & \acc{75.8}{1.1} & \acc{76.7}{0.5} & \acc{81.7}{1.9} & \acc{80.5}{0.4} & \macc{87.1}{0.1} & \acc{88.8}{2.1} & \acc{72.9}{1.7}& \acc{77.6}{1.5}\\

Sketch & \acc{42.6}{2.9} & \acc{45.3}{5.0} & \acc{51.1}{2.6} & \acc{50.2}{3.4} & \acc{50.6}{2.0} & \acc{50.9}{2.4} & \acc{56.8}{4.1} & \acc{53.1}{1.5} & \acc{63.2}{2.5} & \acc{63.9}{5.8} & \acc{51.9}{1.4} & \acc{52.7}{3.1}\\

Infograph & \acc{33.1}{2.8} & \acc{34.7}{3.7} & \acc{35.3}{2.8} & \acc{35.0}{4.0} & \acc{38.3}{1.1} & \acc{38.2}{2.5} & \acc{39.2}{3.7} & \acc{39.7}{2.7} & \acc{42.3}{4.2} & \acc{41.6}{7.1} & \acc{38.5}{2.9} & \acc{37.8}{3.4}\\

Painting & \acc{49.0}{1.7} & \acc{52.5}{3.3} & \acc{56.1}{1.4} & \acc{55.1}{2.5} & \acc{56.2}{0.7} & \acc{59.3}{0.8} & \acc{64.2}{1.8} & \macc{61.8}{0.2} & \acc{69.6}{0.5} & \acc{76.5}{3.0} & \acc{56.4}{1.9}& \acc{59.7}{1.6}\\

Clipart & \acc{47.5}{3.6} & \acc{49.7}{4.8} & \acc{55.5}{3.1} & \acc{54.9}{4.3} & \acc{56.4}{2.6} & \acc{60.4}{2.3} & \acc{63.0}{4.3} & \acc{60.9}{1.8} & \acc{72.7}{1.5} & \acc{67.4}{7.0} & \acc{58.4}{2.2} & \acc{58.8}{3.4}\\

\end{tabular}
\vskip -0.2in
\end{table*}

In order to evaluate on different types and degrees of task distribution shift, in the following experiments, we select a broad range of datasets for $\mathcal{D}^{train}$ and $\mathcal{D}^{test}$. For $\mathcal{D}^{train}$, we choose (1) the train split of \emph{mini}ImageNet~\cite{matchingnet} that contains 38400 images from 64 classes; (2) the train split of ImageNet 1K~\cite{ImageNet} containing more than 1M images from 1000 classes; (3) train+val split of iNaturalist 2018~\cite{iNaturalist}, a fine-grained dataset of plants and animals with a total of more than 450000 training images from 8142 classes. For $\mathcal{D}^{test}$, we choose the test split of \emph{mini}ImageNet, and all evaluation datasets of Meta-dataset~\cite{metadataset}, BSCD-FSL benchmark~\cite{BSCD} and DomainNet~\cite{DomainNet}, for a total of 19 datasets, to ensure adequate coverage of different categories, domains and task granularities.

\subsection{Universal Performance Gains from a Test-time Simple Feature Transformation}
Let $\xx\in\mathbb{R}^D$ denote an image  and $f_\theta(\cdot)$ a feature extractor learned from the training set $\mathcal{D}^{train}$. The $l$-th channel of the feature $\zz=f_\theta(\xx)\in\mathbb{R}^d$ is defined as the $l$-th dimension of $\zz$, i.e., $\{z_i\}_{i=1}^d$ is the set of all $d$ channels. The simple transformation function $\phi_k: [0,+\infty)\rightarrow[0,+\infty)$ that we consider is defined as
%that can consistently improve performance on few-shot classification tasks. We consider the following function:
% \[
% \phi(x) = sign(x)\frac{1}{ln^k(\frac{1}{|x|}+1)},
% \]
\begin{equation}
    \label{simple_transformation}
    \phi_k(\lambda) =\begin{cases} \frac{1}{ln^k(\frac{1}{\lambda}+1)}, &\lambda>0\\
    0, &\lambda=0
    \end{cases}
\end{equation}
where $k>0$ is a hyperparameter. At test time, we simply use this function to transform each channel of image features, i.e.,
\begin{equation}
    \phi_k(\zz)=(\phi_k(z_1), ..., \phi_k(z_d)).
\end{equation}
% We append this transformation function on the top of any \emph{pre-trained} feature extractor $f_\theta(\cdot)$ channel-wisely only at test-time, which essentially applies a non-linear transformation to the \emph{learned representations} just after the global pooling layer. 
When applying this transformation, we transform all image features in the target classification task regardless of whether they are in the support set or query set; any subsequent operation keeps unchanged. Note that this function can only be applied to features taking non-negative values, common in most convolutional neural networks using ReLU as the activation function. We discuss one variant of the function dealing with features having negative values (e.g.,  networks with Leaky ReLU) in Appendix \ref{secattempt}. A plot of this function with various choices of $k$ is shown in Figure~\ref{transformation_plot}. 

Table~\ref{performance} shows the performance gains brought by this transformation on 5-way 5-shot FSL tasks. We test the transformation on representations trained with different algorithms, including (1) the conventional training methods including cross-entropy (CE) and the S2M2 algorithm~\cite{S2M2}, (2) meta-learning methods including ProtoNet~\cite{ProtoNet} (PN), Meta-baseline~\cite{metabaseline}  and MetaOpt~\cite{metaopt}, and (3) MoCo-v2~\cite{MoCo}, a unsupervised contrastive learning method. We test these methods with various backbone networks: Conv-4~\cite{matchingnet} and  four variants of ResNet~\cite{Resnet} including ResNet-12~\cite{tadam}, WRN-28-10~\cite{WRN}, ResNet-50 and SE-ResNet50~\cite{SENet}. We replace Leaky ReLU with ReLU in ResNet-12 to obtain positive features (cause of performance degradation in Table~\ref{performance}). At test-time, we use the  Nearest-Centroid Classifier~\cite{ProtoNet} for CE, linear probing for S2M2 and MoCo-v2, and for meta-learning algorithms we use their own test-time classifier. Training and evaluation details can be found in Appendix~\ref{secdetails}.

\begin{figure}[t]
\vskip 0.2in
\centering
\centerline{\includegraphics[width=0.6\columnwidth]{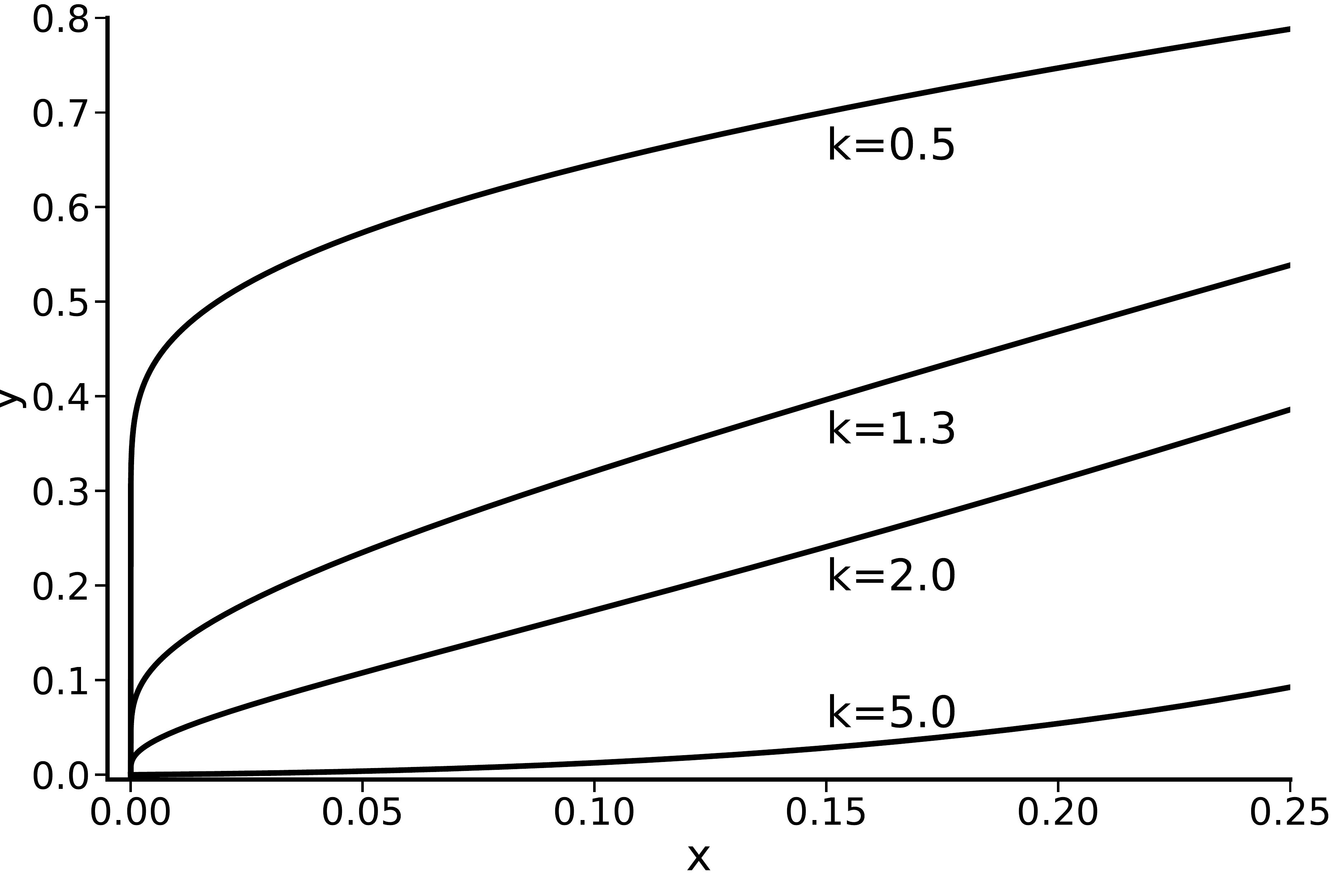}}
\caption{The simple transformation function $\phi_k$ with various choices of $k$.}
\label{transformation_plot}

\vskip -0.2in
\end{figure}
\begin{figure}[t]
% \vskip 0.2in
\centering
\centerline{\includegraphics[width=0.8\columnwidth]{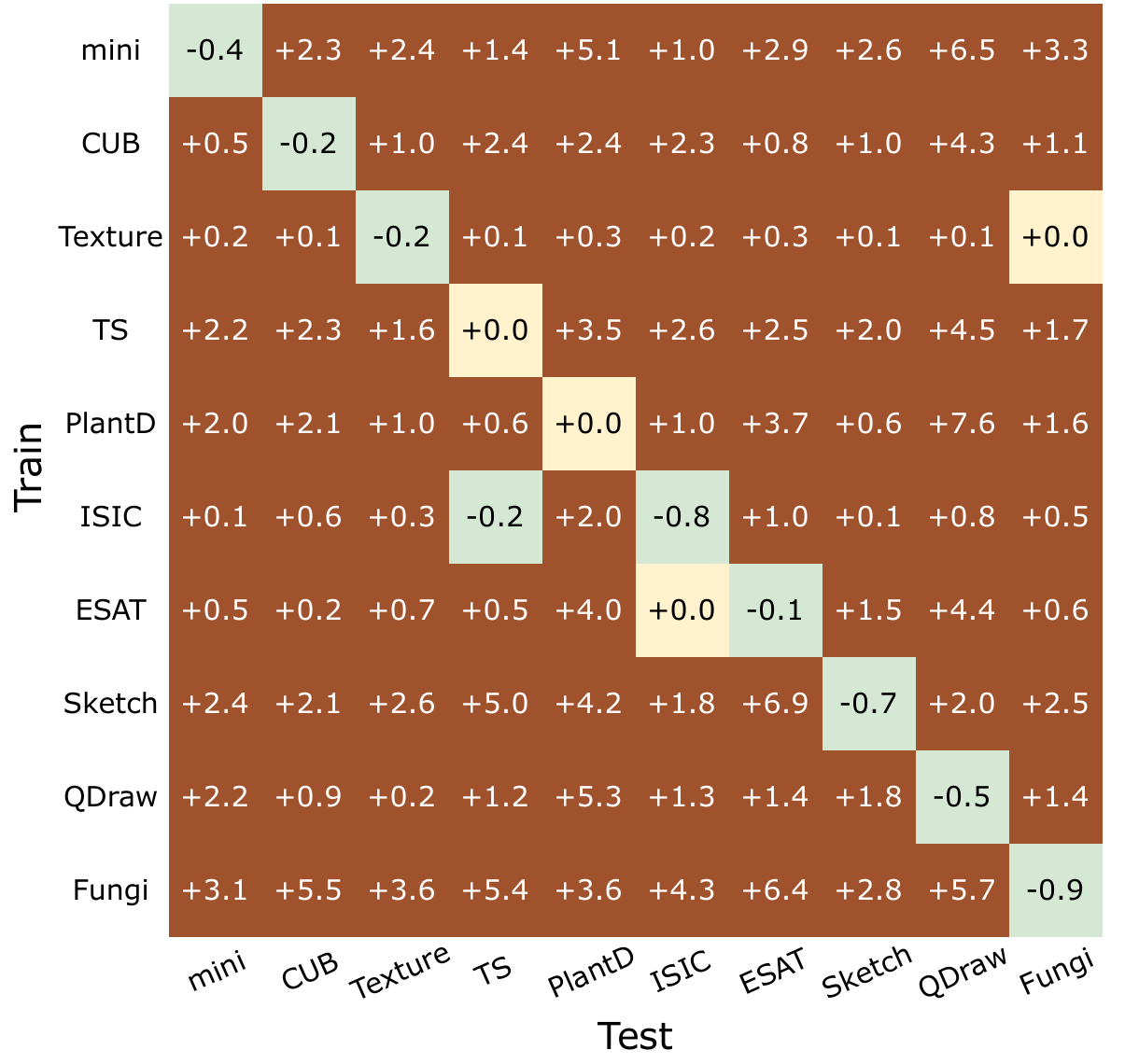}}
\caption{In-distribution~(diagonal) and out-of-distribution~(off-diagonal) performance gains of the simple channel-wise transformation on representations trained with CE. When the test-time dataset equals the training dataset~(diagonal), the categories of images remain the same but test-time images are unseen during training~(as in conventional classification).}
\label{pairwise}

\vskip -0.2in
\end{figure}
% Training and evaluation details can be found in the appendix.

The result shows how this simple feature transformation substantially improves few-shot learning across various algorithms, datasets and architectural choices, with a fixed hyperparameter $k=1.3$~(We show how performance varies with different choices of $k$ in Appendix \ref{seceffectk}). The only exception happens when the test-time task distribution is very similar to a subset of training distribution: training the supervised models on ImageNet and testing on \emph{mini}ImageNet, MSCOCO, Real or Painting\footnote{There are a lot of painting-style images in ImageNet. Contrastive learning~(MoCo) can be seen as an infinitely fine-grained classification task, thus having a relatively large different task distribution shift from training to testing, even on the same dataset.}, or training on iNaturalist and testing on CUB. Is this transformation useful only if there exists task distribution shift between training and testing? To verify this, we train a CE model on each of ten datasets and test on 5-way 5-shot tasks sampled from each dataset. When testing on the training dataset, we evaluate on images not included in training. The results shown in Figure~\ref{pairwise} clearly give evidence that the transformation is beneficial only to few-shot classification with task distribution shift---the performance is improved only when test-time task distribution deviates from training, and this distribution shift includes domain shift~(e.g., from Sketch to QuickDraw), category shift (e.g., from Plant Disease to Fungi) and granularity shift (e.g., from iNaturalist to Plant Disease in Table \ref{performance}).

\section{The Channel Bias Problem }
\label{sec3}
%We have seen that a simple deterministic transformation impressively improves few-shot transferability of vision models with various backbones, algorithms and training datasets. %Then one may ask: what intrinsic reason lies behind?  In this section, we try to give an answer.
In this section, we analyze the simple transformation, which leads us to discover the channel bias problem of visual representations.
Given the transformation function described in Eq.(\ref{simple_transformation}), it can be first noticed that
\begin{gather}
\label{property}
\phi_k'(\lambda)>0, \lim_{\lambda\to 0^{+}}\phi_k'(\lambda)=+\infty,\notag\\
\exists t>0, \quad s.t. \quad \forall \lambda\in(0,t), \phi_k''(\lambda)<0,
\end{gather}
where $t$ is a large value for most $k$, relative to the magnitudes of almost all channels (e.g., when $k=1.3$, $t\approx 0.344$, while most channel values are less than $0.3$). The positiveness of the derivative ensures that the relative relationship between channels will not change,  while the negative second derivative narrows their gaps; the infinite derivative near zero pulls up small channels by a large margin, i.e., $\lim_{\lambda\to 0^+}\frac{\phi_k(\lambda)}{\lambda}=+\infty$. See Appendix \ref{ingredient} for the necessity of all these properties. A clear impact of these properties on features is to make channel distribution smooth: suppress channels with high magnitude, and largely amplify channels with low magnitude. This phenomenon is clearly shown in Figure \ref{feature_visulization}, where we plot mean magnitudes of all 640 feature channels on \emph{mini}ImageNet and PlantDisease, with red ones being the original distribution, blue ones being the transformed distribution. The transformed distribution becomes more uniform.

\begin{figure}[t]
% \vskip 0.2in
\centering
\centerline{\includegraphics[width=1.0\columnwidth]{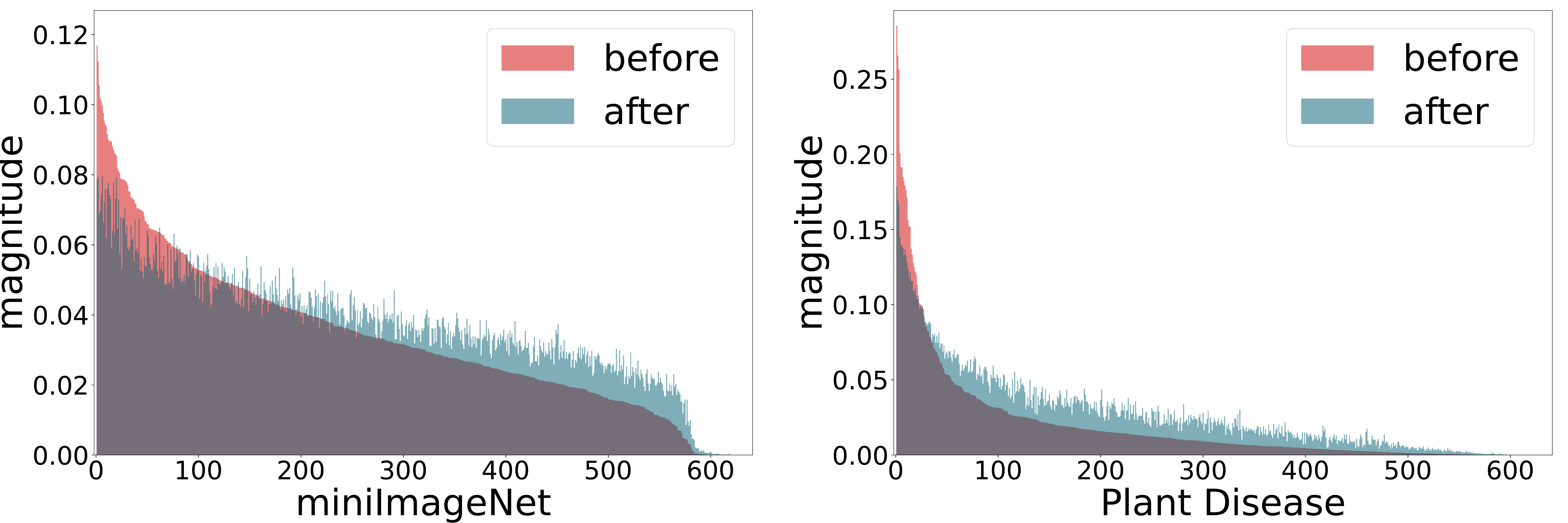}}
\caption{\textbf{Mean magnitudes of feature channels before and after applying the simple transformation.} The feature extractor is trained using PN on the training set of \emph{mini}ImageNet. Left: test set of \emph{mini}ImageNet. Right: The Plant Disease dataset. The change of relative magnitude is due to different variances of channels.}
\label{feature_visulization}

\vskip -0.2in
\end{figure}
Intuitively, different channels have high responses to different features, and a larger Mean Magnitude of a Channel (MMC) implies that the model puts more emphasis on this channel, hoping that this channel is more important for the task at hand. Combining the analysis above with previous experiment results, we conjecture that the MMC of representations should change when testing on novel tasks with a shift in distribution. This meets our intuition that different tasks are likely to be characterized by distinct discriminative features, as shown in the examples of Figure \ref{example_channelchange}.

\subsection{Deriving the Oracle MMC of Any Binary Task}
% To validate that the channel importance across tasks should be different, we directly derive the optimal channel importance for any given classification task, and compare the importance across tasks.

We now wonder how much the  MMC estimated by neural networks in a task deviates from the best MMC or \emph{channel importance} of that task. To achieve this goal, we first derive the optimal MMC for any classification task by multiplying a positive constant to each channel of features, given that we know the first-order and second-order statistics of features. For convenience, we consider the binary classification problem. Specifically, let $\mathcal{D}_1$, $\mathcal{D}_2$ denote probability distributions of two classes over feature space $\mathcal{Z}\subset [0,+\infty)^d$, and $\zz_1\sim\mathcal{D}_1$, $\zz_2\sim\mathcal{D}_2$ denote samples of each class. Let $\uu_1, \uu_2$ and $\SS_1, \SS_2$ denote their means and covariance matrices, respectively. We assume that the channels of features are uncorrelated with each other, i.e., there exist $\ss_1, \ss_2\in[0,+\infty)^d$, \emph{s.t.}  $\SS_1=\mathrm{diag}(\ss_1)$, $\SS_2=\mathrm{diag}(\ss_2)$. The original MMC of the binary task is defined as $\oo^o=(\uu_1+\uu_2)/2$.
We assume that the MMC after adjustment is $\oo\in [0,+\infty)^d$. Let $\hz_1$, $\hz_2$ denote standadized version of $\zz_1$, $\zz_2$ that have unit MMC, i.e., $\widetilde{z}_{1,l}=z_{1,l}/\omega^o_{l},\widetilde{z}_{2,l}=z_{2,l}/\omega^o_{l}\Rightarrow(\wu_{1,l}+\wu_{2,l})/2=1, \forall l\in [d]$ ($[d]$ is equivalent to $\{1,2,...,d\}$). A simple approach to adjust MMC to $\oo$ is to transform features to $\oo\odot \hz_1$ and $\oo\odot\hz_2$ respectively, where $\odot$ denotes the hadamard product. Here, we consider a metric-based classifier. Specifically, a standardized feature $\bm{\hz}$ is classified as the first class if $||\oo\odot (\hz-\hu_1)||_2<||\oo\odot(\hz-\hu_2)||_2$ and otherwise the second class. This classifier is actually the Nearest-Centroid Classifier (NCC)~\cite{ProtoNet} with accurate centroids. Assume that two classes of images are sampled equal times, then the expected misclassification rate of this classifier is
\begin{align}
\mathcal{R} = 
        \frac{1}{2}&[\mathbb{P}_{\zz_1\sim \mathcal{D}_1}(||\oo\odot (\hz_1-\hu_1)||_2>||\oo\odot (\hz_1-\hu_2)||_2)\nonumber\\
        +&\mathbb{P}_{\zz_2\sim \mathcal{D}_2}(||\oo\odot (\hz_2-\hu_2)||_2>||\oo\odot (\hz_2-\hu_1)||_2)].
\end{align}
The following theorem gives an upper bound of the misclassification rate and further gives the \emph{oracle} MMC of any given task.
\begin{proposition}
\label{thm:bigtheorem}
Assume that $\mu_{1,l}\neq\mu_{2,l}$ and $\sigma_{1,l}+\sigma_{2,l}>0$ hold for any $l\in[d]$, then we have
\begin{equation}
    \begin{split}
\mathcal{R}\leq\frac{8\sum_{l=1}^d\omega_l^4(\ws_{1,l}+\ws_{2,l})^2}{(\sum_{l=1}^d\omega_l^2(\wu_{1,l}-\wu_{2,l})^2)^2}.
    \end{split}
\end{equation}
To minimize this upper bound, the adjusted oracle MMC of each channel $\omega_l$ should satisfy:
\begin{equation}
\label{oracle_adjustment}
\omega_l \propto \frac{|\mu_{1,l}-\mu_{2,l}|}{\sigma_{1,l}+\sigma_{2,l}}.
\end{equation}
\end{proposition}

\begin{figure}[t]
% \vskip 0.2in

\centering
\centerline{\includegraphics[width=0.9\linewidth]{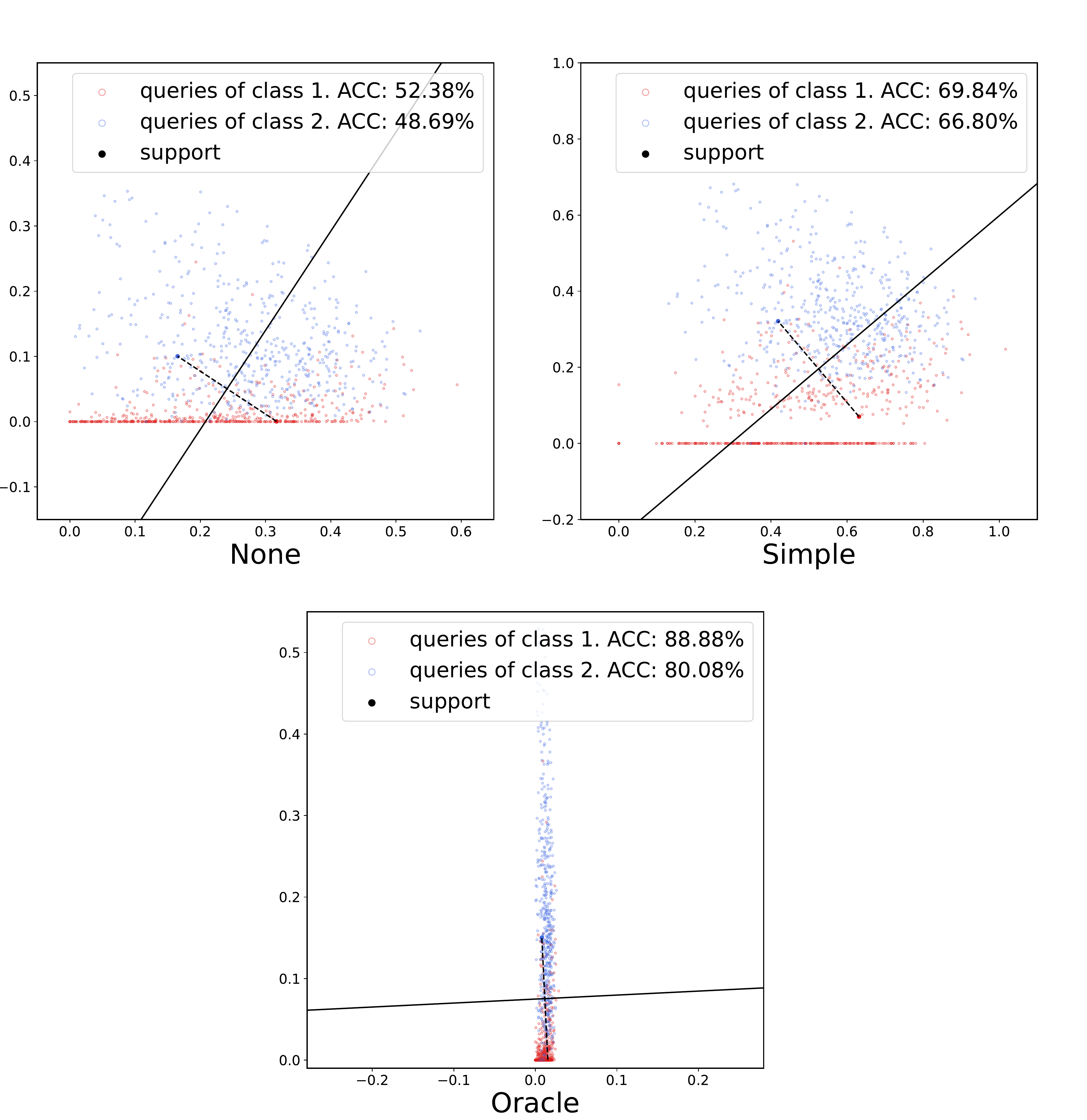}}
% \hfill
\caption{\textbf{Visualization of two channels of image features in two classes of Plant Disease.} The feature extractor is trained using PN on \emph{mini}ImageNet. We visualize a one-shot task with only two channels available for classification. The plot with ``None'' shows the original channels. The plots with ``Simple'' and ``Oracle'' show channels adjusted by the simple and oracle transformation. The per-class accuracy is calculated as the proportion of samples correctly classified by the classification boundary in each class.}

\vskip -0.15in
\label{gaussion_example}
\end{figure}

\begin{figure*}[t]
% \vskip 0.2in

\centering
\centerline{\includegraphics[width=1.0\linewidth]{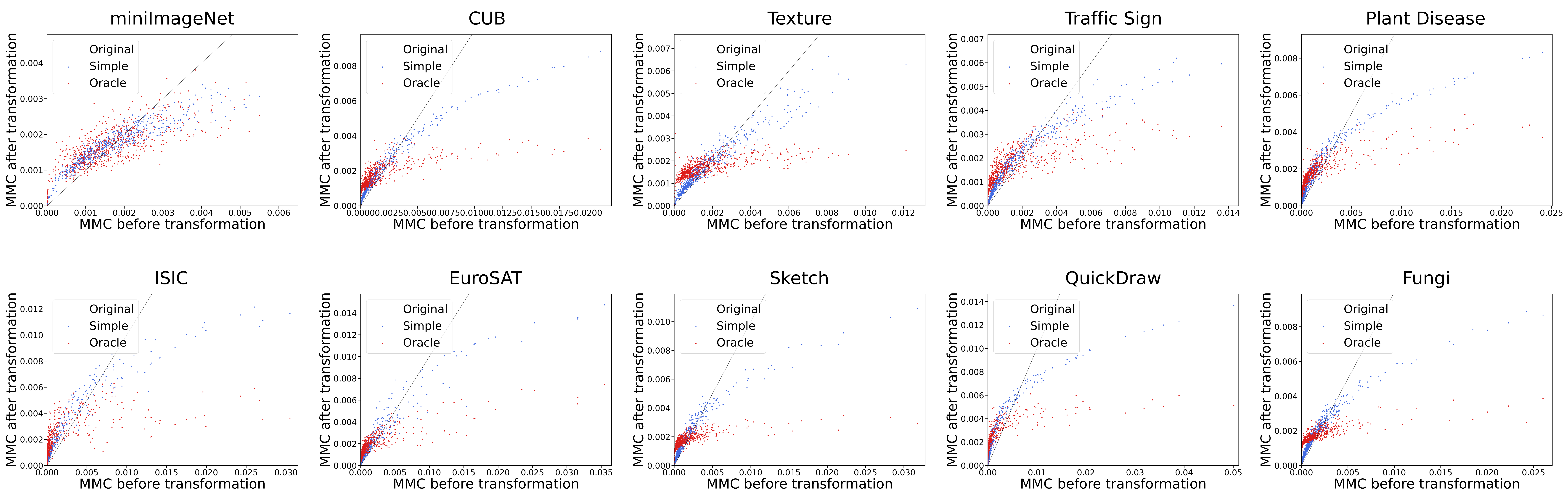}}
% \hfill
\caption{\textbf{Visualization of MMC of ten datasets  $\oo_D$ before and after the use of simple and oracle transformation.} In each plot, a point represents a channel, and the x-axis and y-axis represent the MMC before and after transformation respectively, averaged over all possible binary tasks in the corresponding dataset. For comparison, we also plot the line $y=x$ representing the ``None'' scenario where none of the transformations are applied to features. The feature extractor is trained using PN on \emph{mini}ImageNet.}

\vskip -0.15in
\label{visiualize_optimal}
\end{figure*}

\begin{table*}[t]
\setlength\tabcolsep{4pt}
\footnotesize
% \scriptsize
\caption{The performance gains of the oracle MMC on 5-shot binary classification tasks on various datasets. The derived MMC improves the few-shot performance of both metric and non-metric test-time methods: Nearest-Centroid Classifier (NCC) and Linear Classifier (LC).}
\label{Optimal}
\centering
\begin{tabular}{ccc|cccccccccc|c}
% \Cline{0.6pt}{1-8}
\\[-1em]
\\
 Algorithm & Classifier & Transformation & mini & CUB & Texture & TS & PlantD & ISIC & ESAT & Sketch & QDraw & Fungi & Avg
\\
\multirow{3}{*}{PN}& \multirow{3}{*}{NCC} &None & 90.5 & 80.6 & 80.6 & 85.1 & 89.2 & 65.7 & 86.5 & 71.9 & 82.4 & 74.6 & 80.7
\\
& & Simple& 91.3 & 82.4 & 83.1 & 85.8 & 93.0 & 68.6 & 89.2 & 75.2 & 85.1 & 77.2 & 83.1
\\
& &Oracle  & \textbf{93.1} & \textbf{88.7} & \textbf{87.2} & \textbf{92.4} & \textbf{95.6} & \textbf{69.1} & \textbf{91.5} & \textbf{81.2} & \textbf{89.4} & \textbf{88.4} & \textbf{87.7}
\\ \hline
\multirow{3}{*}{S2M2}&\multirow{3}{*}{LC}&None & 94.0 & 87.1 & 85.7 & 88.7 & 95.0 & 68.7 & 93.5 & 78.7 & 85.5 & 82.8 & 86.0
\\
& &Simple & 94.4 & 88.3 & 87.3 & 91.2 & 96.4 & 72.2 & 93.8 & 81.0 & 89.2 & 84.5 & 87.8
\\
& &Oracle & \textbf{96.3} & \textbf{94.0} & \textbf{90.7} & \textbf{96.1} & \textbf{98.3} & \textbf{72.6} & \textbf{95.2} & \textbf{87.0} & \textbf{93.0} & \textbf{93.3} & \textbf{91.7}
\end{tabular}
\end{table*}

\begin{table*}[t]
\setlength\tabcolsep{2.0pt}
\footnotesize
% \scriptsize
\caption{Three levels of distance  between different MMCs or $l_1$-normalized image features. The first row shows the dataset-level distance between the original MMC of the training
set (mini-train) and each test set; the second row shows the dataset-level distance between the original and oracle MMCs on each dataset; rows 3-6 show the task-level and image-level distances (both amplified by $10^6$ times) between MMCs obtained by simple and oracle transformation or between original MMCs (None) and the MMCs obtained by oracle transformation. The feature extractor is trained using PN on the training set of \emph{mini}ImageNet (mini-train).}
\label{distance}
\centering
\begin{tabular}{ccc|ccccccccccc}
% \Cline{0.6pt}{1-8}
\\[-1em]
\hline
& & &\multicolumn{11}{c}{Test dataset}\\
 Level & Compared dataset & Trans. & mini-train & mini-test & CUB & Texture & TS & PlantD & ISIC & ESAT & Sketch & QDraw & Fungi\\\hline
 \multirow{2}{*}{Dataset} & Train \emph{v.s.} Test & None & - & 0.18 & 1.56 & 0.88 & 1.13 & 1.54 & 2.28 & 1.30 & 1.01 & 1.58 & 0.79\\
 &Test & None \emph{v.s.} Oracle & 0.42 & 0.72 & 3.60 & 1.78 & 4.04 & 3.92 & 3.47 & 5.62 & 4.26 & 3.37 & 3.87\\\hline
 \multirow{2}{*}{Task}&\multirow{2}{*}{Test} & None \emph{v.s.} Oracle & 3.60 & 4.04 & 3.53 & 4.13 & 3.68 & 4.09 & 3.15 & 4.24 & 5.31 & 4.22 & 3.38\\
 & & Simple \emph{v.s.} Oracle & 3.54 & 3.80 & 2.93 & 3.62 & 3.22 & 3.65 & 2.59 & 3.71 & 4.35 & 3.18 & 2.78\\\hline
 \multirow{2}{*}{Image}&\multirow{2}{*}{Test} & None \emph{v.s.} Oracle & 10.52 & 10.65 & 11.53 & 25.20 & 9.88 & 9.75 & 13.04 & 16.33 & 27.36 & 13.46 & 11.39\\
 & & Simple \emph{v.s.} Oracle & 7.98 & 8.14 & 8.69 & 16.74 & 7.06 & 7.34 & 8.43 & 11.22 & 19.50 & 9.32 & 8.71\\\hline
\end{tabular}
\end{table*}

Proofs are given in Appendix \ref{all_proof}. We here use the word ``oracle'' because it is derived using the class statistics of the target dataset, which is not available in few-shot tasks. This derived MMC has an intuitive explanation: if the difference between the means of features from two classes is large but the variances of features from two classes are both small, the single channel can better distinguish the two classes and thus should be emphasized in the classification task. 
In fact, if we further assume $x_{1,l}$ and $x_{2,l}$ are Gaussian-distributed and consider only using the $l$-th channel for classification, then the misclassification error for the $i$-th class ($i=1,2$) is a strictly monotonically decreasing function of $|\mu_{1,l}-\mu_{2,l}|/\sigma_{i,l}$. 

% This  is achieved by adjusting the 

% Through the adjustment of MMC, the variance of axis $x$ becomes smaller, and the variance of axis $y$ becomes larger.

% simple and oracle adjustment both 

% the classification boundary formed by sampling one sample from each class may be incorrect, but 

% through two selected channels

% two selected channels of transferred representations of two classes on Plant Disease . The feature extractor is trained on \emph{mini}ImageNet.

% In Figure 5, we show two channels of data on two fine-grained classes

Table~\ref{Optimal} shows the performance improvement over the simple feature transformation when adjusting the MMC to derived oracle one in each of the real few-shot binary classification tasks. For every sampled binary task in a dataset, we calculate the oracle adjustment based on Eq.~(\ref{oracle_adjustment}); see Appendix \ref{MMC_calculation} for details. The oracle MMC improves performance on all datasets, and always by a large margin. Note that although the oracle MMC is derived using a metric-based classifier, it can also help a linear classifier to boost performance, which will be further discussed in Section 4. The large performance gains using the derived channel importance indicate that the MMC of features on new test-time few-shot task indeed has a large mismatch with ground-truth channel importance. 

To obtain a better understanding, in Figure \ref{gaussion_example}, we visualize image representations of two classes when transferred from \emph{mini}ImageNet to Plant Disease. The two exhibited classes are apples with Apple Scab and Black Rot diseases, respectively. We visualize 2 out of 640 channels in the features, shown as the $x$-axis and $y$-axis in the figure. We select these channels by first selecting a channel that requires a large suppression of MMC ($x$-axis), and then a channel that requires a large increase ($y$-axis). As seen, the $x$-axis channel has a large intra-class variance (
the variances are 0.13 and 0.11 in two classes on the $x$-axis channel, compared to 0.03 and 0.08 on the $y$-axis channel) and a small class mean difference (about 0.03, compared to 0.13 on the $y$-axis channel), so it is hard to distinguish two classes through this channel. By adjusting the mean magnitude of this channel, the simple transformation and oracle adjustment decrease the intra-class variance of the $x$-axis channel, and so decrease its influence on classification. Similarly, the $y$-axis channel can better distinguish two classes due to its relatively larger class mean difference and smaller intra-class variance, so the influence of the $y$-axis channel should be strengthened.

% the influence of $y$-axis channel are strengthened because of its relatively larger class mean difference and smaller intra-class variance. 

% The simple transformation here acts as a prior knowledge that channels with large MMC are likely to have a smaller importance.

% The analysis of the $y$-axis channel is similar, but say the opposite.

% exhibits relatively smaller intra-class variance, and is more discriminative. When sampling one sample from each class to form a linear classifer, the classifer is likely to construct an incorrect classification boundary by wrongly taking the high intra-class variance of the $x$-axis channel as main information for classification. The simple transformation and oracle adjustment correct such blindness to intra-class variance by adjusting the scale of each channel.
% The larger the difference of two means is, and the smaller the two variances are, the better the single channel can discriminate the two classes. We then utilize this weight importance to perform few-shot evaluation, in order to check whether our theoretical insights match practice. In each 2-way task, we calculate the corresponding importance vector $\omega$, multiply it to every feature vector channel-wisely, and use the transformed features for classification. The results are shown in Table . We can clearly see that the method optimally obtained is much better than the simple transformation. Therefore, we can treat it as an oracle method that upper-bounds the performance if we want to adjust the importances of channels. % why not using multi-way?
\subsection{Analysis of Channel Importance}
Next, we take the derived oracle MMC as an approximation of the ground-truth channel importance, and use it to observe how the simple transformation works, as well as how much the channel emphasis of neural networks deviates from the ground-truth channel importance of tasks in each test-time dataset. We define MMC of a dataset $D$ as the average $l_1$-normalized MMCs over all possible binary tasks in that dataset. Specifically, suppose in one dataset $D$ there are $C$ classes, and let $\oo_{ij}$ denote the MMC in the binary task discriminating the $i$-th and $j$-th class. $\overline{\oo_{ij}}=\oo_{ij}/||\oo_{ij}||_1$ normalizes the MMC, such that the $l$-th component of the vector $\overline{\oo_{ij}}$ represents the percentage of channel emphasis on the $l$-th channel. Then the MMC of $D$ is defined as $\oo_D=\overline{\sum_{1\leq i<j\leq C}\overline{\oo_{ij}}}$,
which gives average percentages of channel emphasis over all binary tasks. We visualize the oracle MMC, compared with MMC adjusted by the simple transformation and the original MMC of each dataset in Figure \ref{visiualize_optimal}. A point in each figure represents a channel of the image features, with $x$ and $y$ axis being its MMC of that dataset before and after transformation, respectively. To obtain a more precise understanding, we also want to quantitatively measure difference between different MMCs or image features. To achieve this, given a distance measure $d(\cdot, \cdot)$ (not necessarily a metric), we define three levels of distances: (1) dataset-level distance $d(\oo_{D_a}, \oo_{D_b})$ that measures the distance between MMCs of two datasets (or the same dataset with different transformations); (2) in-dataset task-level distance $\frac{C(C+1)}{2}\sum_{1\leq i<j \leq C}d(\overline{\oo_{ij}^{a}},\overline{\oo_{ij}^{b}})$ that measures average distance between MMCs of all tasks 
from a dataset obtained by different feature transformations, and (3) image-level distance $\frac{1}{|D|}\sum_{i=1}^{|D|}d(\overline{\zz_a^i},\overline{\zz_b^i})$, a more fine-grained one that measures average distance between all $l_1$-normalized image features $\overline{\zz_a^i},\overline{\zz_b^i}$ of dataset $D$ under different feature transformations. For dataset-level distance, we adopt the normalized mean square difference $d(\xx,\bm{y})=\frac{1}{d}\sum_{l=1}^d(x_l-y_l)^2/x_l^2$, since it treats each channel equally w.r.t. to the scale and is sensitive to high deviation. However, for task-level and image-level distance, we choose the mean square difference  $d(\xx,\bm{y})=\frac{1}{d}\sum_{l=1}^d(x_l-y_l)^2$ instead to avoid high variations caused by a single task or image feature that has channels with very small magnitude; see Appendix \ref{distance_clarification} for details. We calculate the distance (1) between the original MMC of the training set (mini-train) and each test set, to see how much neural networks change channel emphasis when faced with novel tasks, (2) between the original and oracle MMC to see how much the changed emphasis is biased on each dataset, and (3) between the simple and oracle MMC of each dataset to see how much the simple transformation alleviates the problem. The results are shown in Table \ref{distance}. 

% Similar plot and results for different algorithms and datasets can be found in the Appendix.

% \texbf{Overconfidence of } 

\textbf{Neural networks are overconfident in previously learned channel importance.} Comparing the first and second rows in Table \ref{distance}, we can see that the adjustment of MMC that the network made on new tasks is far from enough: the distance of original MMCs between train and test set (the first row) is much smaller than that between original and oracle MMCs on the test set. This suggests channels that are important to previously learned tasks are still considered by the neural network to be important for distinguishing new tasks, but in fact, the discriminative channels are very likely to change on new tasks. This can be also observed from each plot in Figure \ref{visiualize_optimal}, where the oracle MMC pushes up channels having small magnitudes and suppresses channels having large magnitudes. The magnitudes of a large number of small-valued channels are amplified 10$\times$ times or more by the oracle MMC, while large-valued channels are suppressed 5$\times$ times or more, and in most datasets originally large-valued channels eventually have similar channel importance to those of originally small-valued channels. The simple transformation, although not being perfect, also regularizes channels due to its smoothing property discussed in Section 3. We call this problem the \emph{channel bias} problem.

\begin{figure}[t]
% \vskip 0.2in
    \centering
\centerline{\includegraphics[width=1.0\linewidth]{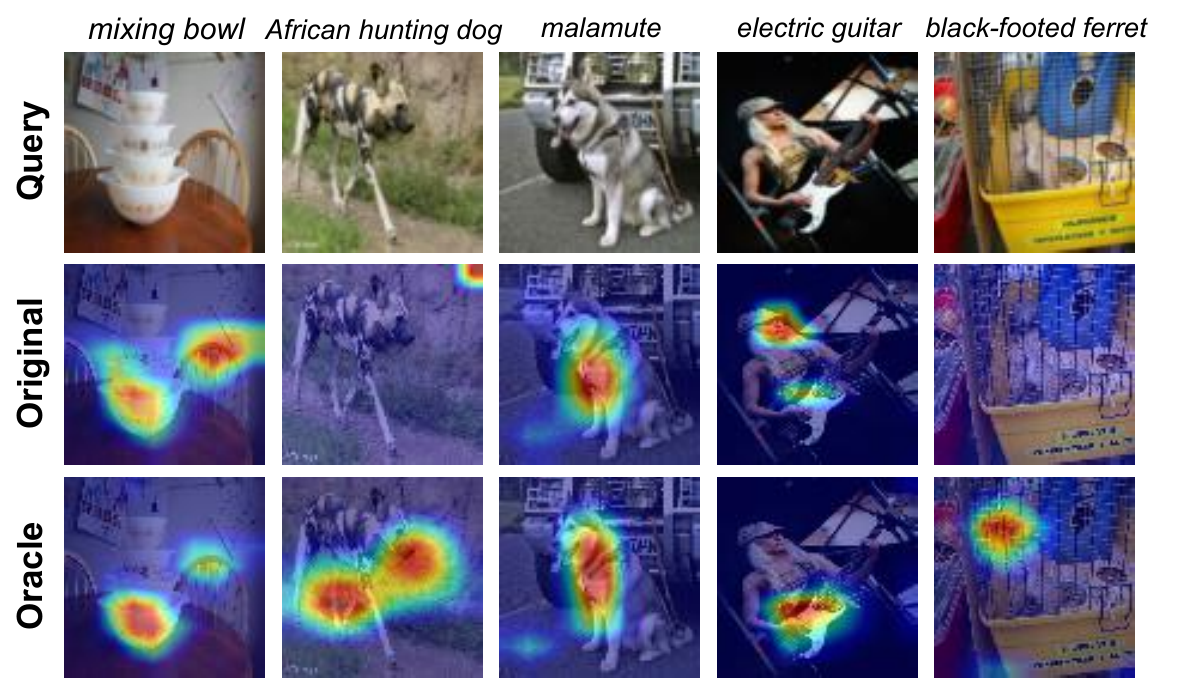}}
% \hfill
\caption{Examples of Grad-Cam~\cite{grad-cam} class activation maps of query samples using PN before and after the oracle adjustment of MMC on binary 5-shot tasks sampled from the test set of \emph{mini}ImageNet.}
\vskip -0.25in
\label{cam}
\end{figure}
\textbf{The channel bias problem diminishes as task distribution shift lessens.} The channel patterns in Figure \ref{visiualize_optimal} on all datasets look similar, except for \emph{mini}ImageNet, whose overall pattern is close to the line $y=x$ representing the original MMCs. There does not exist \emph{dominant} channels when testing on \emph{mini}ImageNet~(The maximum scale of channels is within 0.006), while on other datasets there are channels where the neural network assigns much higher but wrong MMCs which deviate far away from the $y = x$ line. In the second row of Table \ref{distance}, we can also see that the distance between the original and oracle MMCs on \emph{mini}ImageNet, especially on \emph{mini}-train that the model trained on, is much smaller than that on other datasets\footnote{Unnormalized mean square difference ignores critical changes of small-valued channels. This 
is why we do not observe similar phenomenon from the task and image-level difference; see Appendix \ref{distance_clarification} for detailed explanations.}. Since \emph{mini}-test has a similar task distribution with \emph{mini}-train, we can infer that the channel bias is less serious on datasets that have similar task distribution. This explains why in Table \ref{performance} and Figure \ref{pairwise} the simple transformation gets a relatively low improvement when trained on \emph{mini}-train and tested on \emph{mini}-test, and even degrades performance when trained and tested on tasks sampled from the same task distribution.

\textbf{The channel bias problem distracts the neural network from new objects.} In Figure \ref{cam}, we compare some class activation maps before and after the oracle adjustment of MMC. We observe that adjusting channel importance helps the model adjust the attention to the objects responsible for classification using a classifier constructed by only a few support images. This matches observation in previous work~\cite{emergingobject,dissection} that different channels of image representations are responsible for detecting different objects. The task distribution shift makes models confused about which object to focus on, and a proper adjustment of channel emphasis highlights the objects of interest.

\textbf{The simple transformation pushes MMCs towards the oracle ones.} Observing Figure \ref{visiualize_optimal}, it is evident that the simple transformation pushes MMCs towards the oracle ones (compared with the line $y=x$), albeit not perfectly. This observation is further confirmed by the None \emph{v.s.} Oracle and Simple \emph{v.s.} Oracle comparison  of fine-grained task-level and image-level distance shown from the third row to the last row of Table  \ref{distance}. On each of the test-time dataset, the distance between MMCs obtained by simple and oracle transformation is smaller than that bewteen original MMCs and the MMCs obtained by  oracle transformation.
% so here the adjustment of channel importance 

% and here we show that adapting to new objects requires image representations to focus on channels that 

% adjust channel importance.

% the original features often cannot accurately capture new objects responsible for classification using a classifier constructed by only a few support images.   

% These objects are from brand-new classes that 

% struggle to focus on the 

\section{Analysis of the Number of Shots}

We have seen that the channel bias problem is one of the main reasons why image representations cannot generalize well to new few-shot classification tasks. However, two questions remain to be answered: (1) we are still unclear whether this problem is only tied with few-shot image classification. In all previous experiments, we tested on tasks where only 5 labeled images per class are given. What will happen if we have more training examples in the new task? (2) How much will different test-time methods be influenced by the channel bias problem? If we have the opportunity to fine-tune the learned representations, will the proposed simple transformation still work? 

% In this section, we investigate these questions,

% hoping that answering them can give a better understanding of few-shot learning.
\begin{figure}[t]
% \vskip 0.2in

\centering
\centerline{\includegraphics[width=1.0\linewidth]{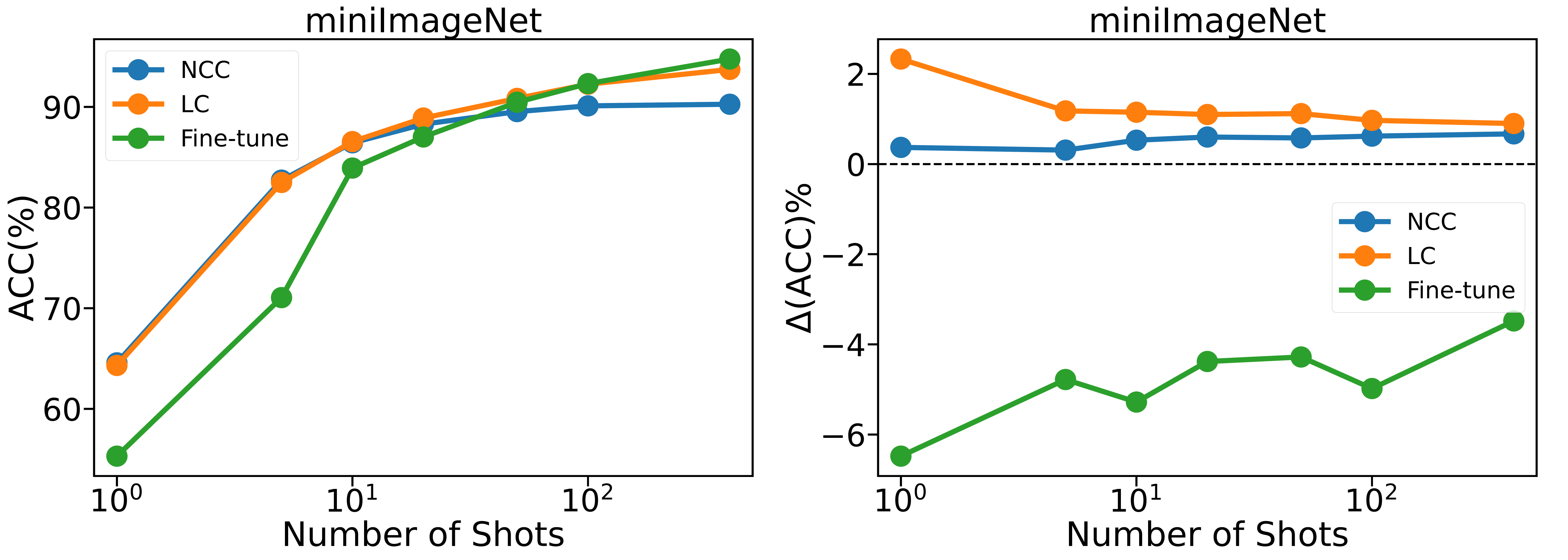}}
% \hfill
\caption{Shot analysis of \emph{mini}ImageNet. Left: performance of different test-time methods. Right: performance gains of the simple transformation using different test-time methods.}
\vskip -0.2in
\label{shot_analysis}
\end{figure}
% \begin{figure*}[t]
% % \vskip 0.2in
% \label{gaussion_example}
% \centering
% \centerline{\includegraphics[width=1.0\linewidth]{Gaussian_example.pdf}}
% % \hfill
% \caption{}
% 
% \vskip -0.2in
% \end{figure*}

In order to give answers to these questions, we conduct shot analysis experiments on three representative test-time methods that are adopted or are the basis of most mainstream few-shot classification algorithms: (1) The metric-based method Nearest-Centroid Classifier (NCC) presented in ProtoNet, which first average image features of each class in the support set to form class centroids and then assign query features to the class of the nearest centroid; (2) Linear Classifier (LC), which trains a linear layer upon learned image features in the support set, and (3) Fine-tuning, which fine-tunes the feature extractor together with the linear layer using images in the support set. The feature extractor is trained using the state-of-the-art S2M2 algorithm on the training set of \emph{mini}ImageNet, and we test it on the test set of \emph{mini}ImageNet using the above three test-time methods with different numbers of labeled images in each class of the support set. The results are shown in Figure \ref{shot_analysis}. We show the original accuracy of all methods, as well as the impact of simple transformation on the performance.

We first take a look at the right plot, which shows the impact of the simple transformation on all the methods. The performance gains on NCC and LC stay at a relatively high value for all tested shots, which is up to 400 labeled images per class. This indicates that the channel bias problem is not only linked to few-shot settings, but also exists in many-shot settings. However, when we have abundant support images, we
have an alternative choice of fine-tuning the feature extractor directly. Fine-tuning methods have the potential to fully resolve the channel bias problem by directly modifying the image representation and rectifying the channel distribution. The right figure shows that the simple transformation does not improve fine-tuning methods, so indeed the channel bias problem has been largely alleviated. In the left figure, the fine-tuning method exhibits its advantages in many-shot setting, but falls short in few-shot settings. Therefore, we can infer that the channel bias problem exists only in the few-shot setting where freezing the feature extractor and building the classifier on learned features becomes a better choice.

We also have another notable observation. While the performance gain of simple transformation on NCC stays around a fixed value, the performance gain on LC decreases with the increase of shots. Thus the channel bias problem is alleviated to some extent in many-shot settings. This is because more labeled data tells the linear classifier
sufficient information about intra-class variance of data,
making it possible to adjust MMC by modifying the
scale of each row of the linear transformation matrix. So Linear Classifier
can stably increase its performance when more labeled data comes in, until no more linear separation can be achieved, and also the time fine-tuning should get into play to adjust the feature space directly.

\section{Discussion and Related Work}

\textbf{Task distribution shift.} Task distribution shift may happen when a model faces category shift, domain shift or granularity shift. Conventional benchmarks of FSL only consider category shift, i.e. the categories are disjoint for training and testing, such as \emph{mini}ImageNet~\cite{matchingnet} and CIFAR-FS~\cite{CIFAR-FS}. In \emph{cross-domain} few-shot learning~\cite{closerlook}, domain shift exists between train and test-time tasks, and several later benchmarks such as BSCD-FSL~\cite{BSCD} and Meta-dataset~\cite{metadataset} both target at such setting. Recently, the shift of granularities of categories has been considered as another type of task distribution shift, and is also called Coarse-to-Fine Few-Shot (C2FS) Learning~\cite{me1, C2FS1,C2FS2}, which trains a model on coarse-labeled 
images and tests on few-shot tasks that aim at distinguishing between fine-grained subclasses of training categories. Our work reveals that all three types of task distribution shift have a similar phenomenon of channel bias problem.

% adjusting the spatial importance and arrangement of features, which 
The influence of task distribution shift on FSL has been firstly studied in~\cite{crosstransformer}. They find that the representation constructed by meta-learning algorithms cannot capture useful discriminative information outside of the training categories. They solve this problem by highlighting the crucial spatial information for classification, using a cross-attention module between support and query features in new tasks. The algorithm COSOC~\cite{rectifying} also considers filtering task-irrelevant spatial information, but it achieves it more directly. They identify image background as harmful information in both training and testing and design a method to remove the background, in order to reduce the difficulty of category transfer. The perspective of our work is different, not focusing on discriminative spatial positions of features, but orthogonally taking inspections on discriminative channel information of features. 

% and it is future work to fully tackle this problem. 

% Designing an algorithm considering alleviating biases from both two orthogonal dimensions is an interesting direction, which we leave as future work.

% The simple transformation presented in this work is obtained by trial and error when we are interested on representations 
% The goal of our work is to reveal and analyze the channel bias problem that vision models face in tasks that have been never seen before. 

% 

% Both two orthogonal dimensions may be crucial for debiasing image features in new tasks and need to be adjusted together, which we leave as future work. 

\textbf{Test-time methods for FSL.} The presented three methods in Section 4 represent three types of mainstream algorithms in FSL. (1) \textbf{Finetuning}---Optimization-based algorithms, originated mainly from MAML~\cite{maml}, optimizes both the learned feature extractor and classifier together at test-time. Most work that fall into this type use meta-learning to train the network~\cite{opt1,opt2,opt3,opt4}. When training adopts conventional supervised approaches, the method turns to resemble transfer learning approaches, and is adopted in~\cite{baseline} and BiT~\cite{BiT}. In the experiments of Section 4, we notice that in few-shot settings, although alleviating channel bias problem, fine-tuning method  performs generally worse and may require very different hyperparameters for different test-time datasets to avoid overfitting, which is impossible to achieve in a realistic few-shot scenario, thus we believe finetuning would not be the best test-time choice. (2) \textbf{NCC}---metric-based algorithms~\cite{matchingnet,ProtoNet,DeepEMD,Cross-Attention,crosstransformer} that aim at learning a well-shaped feature space equipped with a distance metric for comparing the similarity of images, on which the test-time prediction depends. Metric-based methods, as we have shown, benefit from inductive bias given by the metric and thus are widely adopted in state-of-the-art algorithms. (3) \textbf{LC}---most conventionally trained methods adopt LC as the test-time methods~\cite{closerlook,S2M2,allyouneed,neg-cosine,IER}, and two meta-learning algorithms MetaOpt~\cite{metaopt}  and ANIL~\cite{mamlrepresentation} use LC in both training and testing. The importance of a good quality of image representation is mainly figured out from this line of work.

\textbf{Other feature transformations in FSL.} LFT~\cite{LFT} introduces learnable channel-wise feature transformations into training for cross-domain few-shot learning. The transformations are put inside backbone, instead of on top of representations, and are only used at train time, learned in a learning-to-learn fashion using multiple domains of datasets. Z-score transformation upon image representations is introduced in~\cite{z-score} for solving the hubness problem of image representations in FSL.  CCF~\cite{CCF} proposes a variant of variational autoencoder to transform features, which utilizes category relationship between training and test-time classes to rectify the feature distributions.  Feature-wise linear modulation (FiLM)~\cite{film} that turns scaling and shifting coefficients in batch normalization layer (seen as parameters of a linear feature transformation) into dataset- or task-dependent learnable parameters has been adopted in several FSL algorithms~\cite{tadam,CNAPs,template_net,TSA}. The core idea of these methods is to only tune the FiLM modules  at test time in order to reduce overfitting. Thus these methods in some sense belong to finetuning-based methods, and have the potential to perform better than vanilla finetuning in low-shot settings. Contrary to our work, all methods discussed above do not discover or target at the channel bias problem. The most relevant method to our paper may be ConFeSS~\cite{ConFeSS}, a framework that masks task-irrelevant channels in image representations at test time for cross-domain few-shot learning. Our work shows that the success of ConFeSS may be attributed to alleviating the channel bias problem by abandoning overconfident channels when transferred to novel tasks.

% The most relevant work to our work might be~\cite{ConFeSS}

\section{Conclusion}
In this paper, we reveal the channel bias problem in few-shot image classification. The problem can be alleviated by a simple channel-wise feature transformation presented in this work. This transformation, used at test-time without adding any computation overhead, can be applied to most pre-trained convolutional neural networks and few-shot learning algorithms. We show it serves as prior knowledge that regularizes the channel distribution of features. Further analysis, including a derivation of the oracle MMC adjustment, analyzes comprehensively the channel bias problem.  We hope that the channel bias problem revealed in this work, along with analysis of different test-time methods, can provide the community with a better understanding of task distribution shift and representation transfer in few-shot classification, which may in turn help produce better algorithms.

\section*{Acknowledgments}
Special thanks to Qi Yong for providing indispensable spiritual support for the work. We also would like to thank all reviewers for very constructive comments that help us improve the paper. This work was partially supported by the National Key Research and Development Program of China (No. 2018AAA0100204), and a key  program of fundamental research from Shenzhen Science and Technology Innovation Commission (No. JCYJ20200109113403826).

\bibliography{icml}

\begin{thebibliography}{53}
\providecommand{\natexlab}[1]{#1}
\providecommand{\url}[1]{\texttt{#1}}
\expandafter\ifx\csname urlstyle\endcsname\relax
  \providecommand{\doi}[1]{doi: #1}\else
  \providecommand{\doi}{doi: \begingroup \urlstyle{rm}\Url}\fi

\bibitem[Agarwal et~al.(2021)Agarwal, Yurochkin, and Sun]{sensitivity}
Agarwal, M., Yurochkin, M., and Sun, Y.
\newblock On sensitivity of meta-learning to support data.
\newblock In \emph{Advances in Neural Information Processing Systems}, 2021.

\bibitem[Bau et~al.(2017)Bau, Zhou, Khosla, Oliva, and Torralba]{dissection}
Bau, D., Zhou, B., Khosla, A., Oliva, A., and Torralba, A.
\newblock Network dissection: Quantifying interpretability of deep visual
  representations.
\newblock In \emph{Proceedings of the IEEE conference on computer vision and
  pattern recognition}, pp.\  6541--6549, 2017.

\bibitem[Bertinetto et~al.(2019)Bertinetto, Henriques, Torr, and
  Vedaldi]{CIFAR-FS}
Bertinetto, L., Henriques, J.~F., Torr, P. H.~S., and Vedaldi, A.
\newblock Meta-learning with differentiable closed-form solvers.
\newblock In \emph{International Conference on Learning Representations}, 2019.

\bibitem[Bukchin et~al.(2021)Bukchin, Schwartz, Saenko, Shahar, Feris, Giryes,
  and Karlinsky]{C2FS1}
Bukchin, G., Schwartz, E., Saenko, K., Shahar, O., Feris, R., Giryes, R., and
  Karlinsky, L.
\newblock Fine-grained angular contrastive learning with coarse labels.
\newblock In \emph{{IEEE} Conference on Computer Vision and Pattern
  Recognition}, pp.\  8730--8740, 2021.

\bibitem[Cantelli(1929)]{cantelli}
Cantelli, F.~P.
\newblock Sui confini della probabilita.
\newblock In \emph{Atti del Congresso Internazionale dei Matematici: Bologna
  del 3 al 10 de settembre di 1928}, pp.\  47--60, 1929.

\bibitem[Chen et~al.(2019)Chen, Liu, Kira, Wang, and Huang]{closerlook}
Chen, W., Liu, Y., Kira, Z., Wang, Y.~F., and Huang, J.
\newblock A closer look at few-shot classification.
\newblock In \emph{International Conference on Learning Representations}, 2019.

\bibitem[Chen et~al.(2021)Chen, Liu, Xu, Darrell, and Wang]{metabaseline}
Chen, Y., Liu, Z., Xu, H., Darrell, T., and Wang, X.
\newblock Meta-baseline: exploring simple meta-learning for few-shot learning.
\newblock In \emph{Proceedings of the IEEE/CVF International Conference on
  Computer Vision}, pp.\  9062--9071, 2021.

\bibitem[Das et~al.(2022)Das, Yun, and Porikli]{ConFeSS}
Das, D., Yun, S., and Porikli, F.
\newblock Confess: A framework for single source cross-domain few-shot
  learning.
\newblock In \emph{International Conference on Learning Representations}, 2022.

\bibitem[Dhillon et~al.(2020)Dhillon, Chaudhari, Ravichandran, and
  Soatto]{baseline}
Dhillon, G.~S., Chaudhari, P., Ravichandran, A., and Soatto, S.
\newblock A baseline for few-shot image classification.
\newblock In \emph{International Conference on Learning Representations}, 2020.

\bibitem[Doersch et~al.(2020)Doersch, Gupta, and Zisserman]{crosstransformer}
Doersch, C., Gupta, A., and Zisserman, A.
\newblock Crosstransformers: spatially-aware few-shot transfer.
\newblock In \emph{Advances in Neural Information Processing Systems}, 2020.

\bibitem[Fei et~al.(2021)Fei, Gao, Lu, and Xiang]{z-score}
Fei, N., Gao, Y., Lu, Z., and Xiang, T.
\newblock Z-score normalization, hubness, and few-shot learning.
\newblock In \emph{Proceedings of the IEEE/CVF International Conference on
  Computer Vision}, pp.\  142--151, 2021.

\bibitem[Finn et~al.(2017)Finn, Abbeel, and Levine]{maml}
Finn, C., Abbeel, P., and Levine, S.
\newblock Model-agnostic meta-learning for fast adaptation of deep networks.
\newblock In \emph{Proceedings of the 34th International Conference on Machine
  Learning}, pp.\  1126--1135, 2017.

\bibitem[Guo et~al.(2020)Guo, Codella, Karlinsky, Codella, Smith, Saenko,
  Rosing, and Feris]{BSCD}
Guo, Y., Codella, N., Karlinsky, L., Codella, J.~V., Smith, J.~R., Saenko, K.,
  Rosing, T., and Feris, R.
\newblock A broader study of cross-domain few-shot learning.
\newblock In \emph{European Conference on Computer Vision}, pp.\  124--141,
  2020.

\bibitem[He et~al.(2016)He, Zhang, Ren, and Sun]{Resnet}
He, K., Zhang, X., Ren, S., and Sun, J.
\newblock Deep residual learning for image recognition.
\newblock In \emph{{IEEE} Conference on Computer Vision and Pattern
  Recognition}, pp.\  770--778, 2016.

\bibitem[He et~al.(2020)He, Fan, Wu, Xie, and Girshick]{MoCo}
He, K., Fan, H., Wu, Y., Xie, S., and Girshick, R.
\newblock Momentum contrast for unsupervised visual representation learning.
\newblock In \emph{Proceedings of the IEEE/CVF Conference on Computer Vision
  and Pattern Recognition}, pp.\  9729--9738, 2020.

\bibitem[Horn et~al.(2018)Horn, Aodha, Song, Cui, Sun, Shepard, Adam, Perona,
  and Belongie]{iNaturalist}
Horn, G.~V., Aodha, O.~M., Song, Y., Cui, Y., Sun, C., Shepard, A., Adam, H.,
  Perona, P., and Belongie, S.~J.
\newblock The inaturalist species classification and detection dataset.
\newblock In \emph{{IEEE} Conference on Computer Vision and Pattern
  Recognition}, pp.\  8769--8778, 2018.

\bibitem[Hou et~al.(2019)Hou, Chang, Ma, Shan, and Chen]{Cross-Attention}
Hou, R., Chang, H., Ma, B., Shan, S., and Chen, X.
\newblock Cross attention network for few-shot classification.
\newblock In \emph{Advances in Neural Information Processing Systems}, pp.\
  4005--4016, 2019.

\bibitem[Hu et~al.(2018)Hu, Shen, and Sun]{SENet}
Hu, J., Shen, L., and Sun, G.
\newblock Squeeze-and-excitation networks.
\newblock In \emph{Proceedings of the IEEE conference on computer vision and
  pattern recognition}, pp.\  7132--7141, 2018.

\bibitem[Kolesnikov et~al.(2020)Kolesnikov, Beyer, Zhai, Puigcerver, Yung,
  Gelly, and Houlsby]{BiT}
Kolesnikov, A., Beyer, L., Zhai, X., Puigcerver, J., Yung, J., Gelly, S., and
  Houlsby, N.
\newblock Big transfer (bit): General visual representation learning.
\newblock In \emph{European Conference on Computer Vision}, pp.\  491--507,
  2020.

\bibitem[Krizhevsky et~al.(2012)Krizhevsky, Sutskever, and Hinton]{Alexnet}
Krizhevsky, A., Sutskever, I., and Hinton, G.~E.
\newblock Imagenet classification with deep convolutional neural networks.
\newblock In \emph{Advances in Neural Information Processing Systems}, pp.\
  1106--1114, 2012.

\bibitem[Lee et~al.(2019)Lee, Maji, Ravichandran, and Soatto]{metaopt}
Lee, K., Maji, S., Ravichandran, A., and Soatto, S.
\newblock Meta-learning with differentiable convex optimization.
\newblock In \emph{{IEEE} Conference on Computer Vision and Pattern
  Recognition}, pp.\  10657--10665, 2019.

\bibitem[Li et~al.(2022)Li, Liu, and Bilen]{TSA}
Li, W.-H., Liu, X., and Bilen, H.
\newblock Cross-domain few-shot learning with task-specific adapters.
\newblock In \emph{Proceedings of the IEEE/CVF Conference on Computer Vision
  and Pattern Recognition}, pp.\  7161--7170, 2022.

\bibitem[Liu et~al.(2020)Liu, Cao, Lin, Li, Zhang, Long, and Hu]{neg-cosine}
Liu, B., Cao, Y., Lin, Y., Li, Q., Zhang, Z., Long, M., and Hu, H.
\newblock Negative margin matters: Understanding margin in few-shot
  classification.
\newblock In \emph{European Conference on Computer Vision}, pp.\  438--455,
  2020.

\bibitem[Luo et~al.(2021{\natexlab{a}})Luo, Chen, Wen, Pan, and Xu]{me1}
Luo, X., Chen, Y., Wen, L., Pan, L., and Xu, Z.
\newblock Boosting few-shot classification with view-learnable contrastive
  learning.
\newblock In \emph{IEEE International Conference on Multimedia and Expo
  (ICME)}, pp.\  1--6, 2021{\natexlab{a}}.

\bibitem[Luo et~al.(2021{\natexlab{b}})Luo, Wei, Wen, Yang, Xie, Xu, and
  Tian]{rectifying}
Luo, X., Wei, L., Wen, L., Yang, J., Xie, L., Xu, Z., and Tian, Q.
\newblock Rectifying the shortcut learning of background for few-shot learning.
\newblock In \emph{Advances in Neural Information Processing Systems},
  2021{\natexlab{b}}.

\bibitem[Mangla et~al.(2020)Mangla, Singh, Sinha, Kumari, Balasubramanian, and
  Krishnamurthy]{S2M2}
Mangla, P., Singh, M., Sinha, A., Kumari, N., Balasubramanian, V.~N., and
  Krishnamurthy, B.
\newblock Charting the right manifold: Manifold mixup for few-shot learning.
\newblock In \emph{{IEEE} Winter Conference on Applications of Computer
  Vision}, pp.\  2207--2216, 2020.

\bibitem[Oreshkin et~al.(2018)Oreshkin, L{\'{o}}pez, and Lacoste]{tadam}
Oreshkin, B.~N., L{\'{o}}pez, P.~R., and Lacoste, A.
\newblock {TADAM:} task dependent adaptive metric for improved few-shot
  learning.
\newblock In Bengio, S., Wallach, H.~M., Larochelle, H., Grauman, K.,
  Cesa{-}Bianchi, N., and Garnett, R. (eds.), \emph{Advances in Neural
  Information Processing Systems}, pp.\  719--729, 2018.

\bibitem[Park \& Oliva(2019)Park and Oliva]{opt4}
Park, E. and Oliva, J.~B.
\newblock Meta-curvature.
\newblock In \emph{Advances in Neural Information Processing Systems}, pp.\
  3309--3319, 2019.

\bibitem[Pedregosa et~al.(2011)Pedregosa, Varoquaux, Gramfort, Michel, Thirion,
  Grisel, Blondel, Prettenhofer, Weiss, Dubourg, et~al.]{scikit}
Pedregosa, F., Varoquaux, G., Gramfort, A., Michel, V., Thirion, B., Grisel,
  O., Blondel, M., Prettenhofer, P., Weiss, R., Dubourg, V., et~al.
\newblock Scikit-learn: Machine learning in python.
\newblock \emph{the Journal of machine Learning research}, pp.\  2825--2830,
  2011.

\bibitem[Peng et~al.(2019)Peng, Bai, Xia, Huang, Saenko, and Wang]{DomainNet}
Peng, X., Bai, Q., Xia, X., Huang, Z., Saenko, K., and Wang, B.
\newblock Moment matching for multi-source domain adaptation.
\newblock In \emph{Proceedings of the IEEE International Conference on Computer
  Vision}, pp.\  1406--1415, 2019.

\bibitem[Perez et~al.(2018)Perez, Strub, De~Vries, Dumoulin, and
  Courville]{film}
Perez, E., Strub, F., De~Vries, H., Dumoulin, V., and Courville, A.
\newblock Film: Visual reasoning with a general conditioning layer.
\newblock In \emph{Proceedings of the AAAI Conference on Artificial
  Intelligence}, volume~32, 2018.

\bibitem[Raghu et~al.(2020)Raghu, Raghu, Bengio, and
  Vinyals]{mamlrepresentation}
Raghu, A., Raghu, M., Bengio, S., and Vinyals, O.
\newblock Rapid learning or feature reuse? towards understanding the
  effectiveness of {MAML}.
\newblock In \emph{International Conference on Learning Representations}, 2020.

\bibitem[Rajeswaran et~al.(2019)Rajeswaran, Finn, Kakade, and Levine]{opt2}
Rajeswaran, A., Finn, C., Kakade, S.~M., and Levine, S.
\newblock Meta-learning with implicit gradients.
\newblock In \emph{Advances in Neural Information Processing Systems}, pp.\
  113--124, 2019.

\bibitem[Requeima et~al.(2019)Requeima, Gordon, Bronskill, Nowozin, and
  Turner]{CNAPs}
Requeima, J., Gordon, J., Bronskill, J., Nowozin, S., and Turner, R.~E.
\newblock Fast and flexible multi-task classification using conditional neural
  adaptive processes.
\newblock \emph{Advances in Neural Information Processing Systems}, 32, 2019.

\bibitem[Rizve et~al.(2021)Rizve, Khan, Khan, and Shah]{IER}
Rizve, M.~N., Khan, S.~H., Khan, F.~S., and Shah, M.
\newblock Exploring complementary strengths of invariant and equivariant
  representations for few-shot learning.
\newblock In \emph{{IEEE} Conference on Computer Vision and Pattern
  Recognition}, pp.\  10836--10846, 2021.

\bibitem[Russakovsky et~al.(2015)Russakovsky, Deng, Su, Krause, Satheesh, Ma,
  Huang, Karpathy, Khosla, Bernstein, Berg, and Li]{ImageNet}
Russakovsky, O., Deng, J., Su, H., Krause, J., Satheesh, S., Ma, S., Huang, Z.,
  Karpathy, A., Khosla, A., Bernstein, M.~S., Berg, A.~C., and Li, F.
\newblock Imagenet large scale visual recognition challenge.
\newblock In \emph{IJCV}, volume 115, pp.\  211--252, 2015.

\bibitem[Rusu et~al.(2019)Rusu, Rao, Sygnowski, Vinyals, Pascanu, Osindero, and
  Hadsell]{opt1}
Rusu, A.~A., Rao, D., Sygnowski, J., Vinyals, O., Pascanu, R., Osindero, S.,
  and Hadsell, R.
\newblock Meta-learning with latent embedding optimization.
\newblock In \emph{International Conference on Learning Representations}, 2019.

\bibitem[Selvaraju et~al.(2017)Selvaraju, Cogswell, Das, Vedantam, Parikh, and
  Batra]{grad-cam}
Selvaraju, R.~R., Cogswell, M., Das, A., Vedantam, R., Parikh, D., and Batra,
  D.
\newblock Grad-cam: Visual explanations from deep networks via gradient-based
  localization.
\newblock In \emph{Proceedings of the IEEE international conference on computer
  vision}, pp.\  618--626, 2017.

\bibitem[Snell et~al.(2017)Snell, Swersky, and Zemel]{ProtoNet}
Snell, J., Swersky, K., and Zemel, R.~S.
\newblock Prototypical networks for few-shot learning.
\newblock In \emph{Advances in Neural Information Processing Systems}, pp.\
  4077--4087, 2017.

\bibitem[Tian et~al.(2020)Tian, Wang, Krishnan, Tenenbaum, and
  Isola]{allyouneed}
Tian, Y., Wang, Y., Krishnan, D., Tenenbaum, J.~B., and Isola, P.
\newblock Rethinking few-shot image classification: {A} good embedding is all
  you need?
\newblock In \emph{European Conference on Computer Vision}, pp.\  266--282,
  2020.

\bibitem[Triantafillou et~al.(2020)Triantafillou, Zhu, Dumoulin, Lamblin, Evci,
  Xu, Goroshin, Gelada, Swersky, Manzagol, and Larochelle]{metadataset}
Triantafillou, E., Zhu, T., Dumoulin, V., Lamblin, P., Evci, U., Xu, K.,
  Goroshin, R., Gelada, C., Swersky, K., Manzagol, P., and Larochelle, H.
\newblock Meta-dataset: {A} dataset of datasets for learning to learn from few
  examples.
\newblock In \emph{International Conference on Learning Representations}, 2020.

\bibitem[Triantafillou et~al.(2021)Triantafillou, Larochelle, Zemel, and
  Dumoulin]{template_net}
Triantafillou, E., Larochelle, H., Zemel, R.~S., and Dumoulin, V.
\newblock Learning a universal template for few-shot dataset generalization.
\newblock In \emph{Proceedings of the 38th International Conference on Machine
  Learning}, pp.\  10424--10433, 2021.

\bibitem[Tseng et~al.(2020)Tseng, Lee, Huang, and Yang]{LFT}
Tseng, H., Lee, H., Huang, J., and Yang, M.
\newblock Cross-domain few-shot classification via learned feature-wise
  transformation.
\newblock In \emph{International Conference on Learning Representations}, 2020.

\bibitem[Tukey et~al.(1977)]{tukey}
Tukey, J.~W. et~al.
\newblock \emph{Exploratory data analysis}, volume~2.
\newblock Reading, MA, 1977.

\bibitem[Vinyals et~al.(2016)Vinyals, Blundell, Lillicrap, Kavukcuoglu, and
  Wierstra]{matchingnet}
Vinyals, O., Blundell, C., Lillicrap, T., Kavukcuoglu, K., and Wierstra, D.
\newblock Matching networks for one shot learning.
\newblock In \emph{Advances in Neural Information Processing Systems}, pp.\
  3630--3638, 2016.

\bibitem[Xu et~al.(2021)Xu, Pan, Luo, Pei, and Xu]{CCF}
Xu, J., Pan, X., Luo, X., Pei, W., and Xu, Z.
\newblock Exploring category-correlated feature for few-shot image
  classification.
\newblock \emph{arXiv preprint arXiv:2112.07224}, 2021.

\bibitem[Yang et~al.(2021{\natexlab{a}})Yang, Yang, and Chen]{C2FS2}
Yang, J., Yang, H., and Chen, L.
\newblock Towards cross-granularity few-shot learning: Coarse-to-fine
  pseudo-labeling with visual-semantic meta-embedding.
\newblock In \emph{{ACM} Multimedia Conference}, pp.\  3005--3014,
  2021{\natexlab{a}}.

\bibitem[Yang et~al.(2021{\natexlab{b}})Yang, Liu, and Xu]{free_lunch}
Yang, S., Liu, L., and Xu, M.
\newblock Free lunch for few-shot learning: Distribution calibration.
\newblock In \emph{International Conference on Learning Representations},
  2021{\natexlab{b}}.

\bibitem[Ye \& Chao(2022)Ye and Chao]{howtotrainmaml}
Ye, H. and Chao, W.
\newblock How to train your maml to excel in few-shot classification.
\newblock In \emph{International Conference on Learning Representations}, 2022.

\bibitem[Zagoruyko \& Komodakis(2016)Zagoruyko and Komodakis]{WRN}
Zagoruyko, S. and Komodakis, N.
\newblock Wide residual networks.
\newblock In \emph{Proceedings of the British Machine Vision Conference}, 2016.

\bibitem[Zhang et~al.(2020)Zhang, Cai, Lin, and Shen]{DeepEMD}
Zhang, C., Cai, Y., Lin, G., and Shen, C.
\newblock Deepemd: Few-shot image classification with differentiable earth
  mover's distance and structured classifiers.
\newblock In \emph{{IEEE/CVF} Conference on Computer Vision and Pattern
  Recognition}, pp.\  12200--12210, 2020.

\bibitem[Zhou et~al.(2015)Zhou, Khosla, Lapedriza, Oliva, and
  Torralba]{emergingobject}
Zhou, B., Khosla, A., Lapedriza, {\`{A}}., Oliva, A., and Torralba, A.
\newblock Object detectors emerge in deep scene cnns.
\newblock In \emph{International Conference on Learning Representations}, 2015.

\bibitem[Zintgraf et~al.(2019)Zintgraf, Shiarlis, Kurin, Hofmann, and
  Whiteson]{opt3}
Zintgraf, L.~M., Shiarlis, K., Kurin, V., Hofmann, K., and Whiteson, S.
\newblock Fast context adaptation via meta-learning.
\newblock In \emph{Proceedings of the 36th International Conference on Machine
  Learning}, pp.\  7693--7702, 2019.

\end{thebibliography}
\bibliographystyle{icml2022}

\newpage
\appendix
\onecolumn

\section{Proof of Proposition 3.1}
\label{all_proof}
\begin{lemma}
\label{cantelli}
(\textbf{Cantelli's inequality}~\cite{cantelli}) Let $X$ be a random variable with finite expected value $\mu$ and finite non-zero variance $\sigma^2$. Then for any $k>0$,
\begin{equation}
    \mathbb{P}(X-\mu\geq k\sigma)\leq \frac{1}{1+k^2}.
\end{equation}
\end{lemma}

\begin{lemma}
\label{minimum}
Let $a_i>0, b_i>0,i=1,...,D$. Define $f:[0,+\infty)^D/\{\mathbf{0}\}\rightarrow \mathbb{R}$ by 
\begin{equation}
f(\xx)=\frac{\sum_{i=1}^Db_ix_i^2}{(\sum_{i=1}^Da_ix_i)^2},
\end{equation}
then 
\begin{equation}
\min\limits_{\xx}f(\xx)=\frac{1}{\sum_{i=1}^D\frac{a_i^2}{b_i}}.
\end{equation}
The mimimum value is reached when there exists a constant $c>0$, such that $\forall i\in [D], x_i=\frac{a_ic}{b_i}$.
\end{lemma}
\begin{proof}
We show it by induction on dimension $D$. Denote the domain of $f$ by $U$, i.e., $U=[0,+\infty)^D/\{\mathbf{0}\}$.

When $D=1$, $f(x)\equiv\frac{b_1}{a_1^2}$ is a constant, so the result holds.

Assume that when $D\leq k$, the result holds. We now prove that when $D=k+1$, the result holds. It is obvious that $\forall c>0, f(c\xx)=f(\xx)$, thus it suffices to find a minimum point in $\overline{\mathcal{S}}=\{\xx|a\leq||\xx||_2\leq b\;\mathrm{and}\;\xx\in U\}$ for any chosen $0<a<b$. Since $\overline{\mathcal{S}}$ is a closed set and $f$ is continuous, the minimum point exists. The minimum point either lies on the hyperspheres: $\partial S=\{\xx|||\xx||_2=a\;\mathrm{and}\;\xx\in U\}\cup\{\xx|||\xx||_2=b\;\mathrm{and}\;\xx\in U\}$ or in between: $\mathcal{S}=\{\xx|a<||\xx||_2<b\;\mathrm{and}\;\xx\in U\}$. If there exists a minimum point on one of the hyperspheres, say, the outer hypersphere $\{\xx|||\xx||_2=b\;\mathrm{and}\;\xx\in U\}$, then there exists another minimum point $\frac{(a+b)\xx}{2b}\in \mathcal{S}$ (or $\frac{(a+b)\xx}{2a}\in \mathcal{S}$ for the inner hypersphere). Thus it suffices to find the minimum point in $\mathcal{S}$.

Let $\mathcal{S}_{/i}=\{\xx|\xx\in\mathcal{S}\;\mathrm{and}\;x_i=0\}$ and $\mathcal{S}_{>0}=\{\xx|\xx\in\mathcal{S}\;\mathrm{and}\;x_i>0\;\mathrm{for}\;\mathrm{all}\;i\in[k+1]\}$. We have $\mathcal{S}= (\cup_{i=1}^{k+1}\mathcal{S}_{/i})\cup\mathcal{S}_{>0}$. If $\xx\in \mathcal{S}_{/i}$, then 
\begin{equation}
    f(\xx) = \frac{\sum_{j=1}^{i-1}b_jx_j^2+\sum_{j=i+1}^{k+1}b_jx_j^2}{(\sum_{j=1}^{i-1}a_jx_j+\sum_{j=i+1}^{k+1}a_jx_j)^2},
\end{equation}
which can be seen as a function with input dimension $k$. Thus from the induction, we have
\begin{equation}
\label{lasttwo}
\min\limits_{\xx\in \mathcal{S}_{/i}}f(\xx)=\frac{1}{\sum_{j=1}^{i-1}\frac{a_j^2}{b_j}+\sum_{j=i+1}^{k+1}\frac{a_j^2}{b_j}}.
\end{equation}
Next, we handle the setting when $\xx\in\mathcal{S}_{>0}$, i.e., find all possible extreme points of $f$ inside $\mathcal{S}_{>0}$. Note that 
\begin{equation}
    \frac{\partial f}{\partial x_i}=\frac{2(b_ix_i\sum_{j=1}^{k+1}a_jx_j-a_i\sum_{j=1}^{k+1}b_jx_j^2)}{(\sum_{j=1}^{k+1}a_jx_j)^3},
\end{equation}
then an extreme point $\xx$ must satisfy
\begin{equation}
    b_ix_i\sum_{j=1}^{k+1}a_jx_j-a_i\sum_{j=1}^{k+1}b_jx_j^2=0, \forall i\in[k+1],
\end{equation}
which is equivalent to
\begin{equation}
\label{last}
x_i=\frac{a_i\sum_{j=1}^{k+1}b_jx_j^2}{b_i\sum_{j=1}^{k+1}a_jx_j}=(\frac{\sum_{j=1}^{k+1}b_jx_j^2}{\sum_{j=1}^{k+1}a_jx_j})\frac{a_i}{b_i}, \forall i\in[k+1].
\end{equation}
Thus $x_i=c\frac{a_i}{b_i}$, where $c = \frac{\sum_{j=1}^{k+1}b_jx_j^2}{\sum_{j=1}^{k+1}a_jx_j}$. Furthermore, it is easy to show that
$x_i=c\frac{a_i}{b_i}$ satisfies Eq.~(\ref{last}) for any $c>0$. Denote any of the points satisfying this property as $\xx^*$, then
\begin{equation}
    \label{lastzero}
    f(\xx^*)=\frac{1}{\sum_{i=1}^{k+1}\frac{a_i^2}{b_i}}.
\end{equation}
Comparing Eq.~(\ref{lasttwo}) and Eq.~(\ref{lastzero}), it can be seen that

\begin{equation}
    \label{final_minimum}
    f(\xx^*)< \min\limits_{\xx\in \mathcal{S}_{/i}}f(\xx), \forall i\in[k+1].
\end{equation}
Finally, note that $\partial S$ and  $\{\mathcal{S}_{/i}\}_{i=1}^{k+1}$ constitute the boundary of $\mathcal{S}_{>0}$, thus Eq.~(\ref{final_minimum}) and earlier discussion about $\partial S$ indicate that $\xx^*$ is the minimum point of $f$, as desired.
\end{proof}

\textbf{Proposition 3.1.}
\emph{Assume that $\mu_{1,l}\neq\mu_{2,l}$ and $\sigma_{1,l}+\sigma_{2,l}>0$ hold for any $l\in[d]$, then we have}
\begin{equation}
    \begin{split}
\mathcal{R}\leq\frac{8\sum_{l=1}^d\omega_l^4(\ws_{1,l}+\ws_{2,l})^2}{(\sum_{l=1}^d\omega_l^2(\wu_{1,l}-\wu_{2,l})^2)^2}.
    \end{split}
\end{equation}
\emph{To minimize this upper bound, the adjusted oracle MMC of each channel $\omega_l$ should satisfy:}
\begin{equation}
\omega_l \propto \frac{|\mu_{1,l}-\mu_{2,l}|}{\sigma_{1,l}+\sigma_{2,l}}.
\end{equation}
\begin{proof}
    \begin{align}
        \mathcal{R} &= 
        \frac{1}{2}[\mathbb{P}_{\zz_1\sim D_1}(||\oo\odot (\hz_1-\hu_1)||_2>||\oo\odot (\hz_1-\hu_2)||_2)
        +\mathbb{P}_{\zz_2\sim D_2}(||\oo\odot (\hz_2-\hu_2)||_2>||\oo\odot (\hz_2-\hu_1)||_2)]\nonumber\\
        &=\frac{1}{2}[\mathbb{P}_{\zz_1\sim D_1}(\sum_{l=1}^d\omega_l^2[(\wz_{1,l}-\wu_{1,l})^2-(\wz_{1,l}-\wu_{2,l})^2]>0)+\mathbb{P}_{\zz_2\sim D_2}(\sum_{l=1}^d\omega_l^2[(\wz_{2,l}-\wu_{2,l})^2-(\wz_{2,l}-\wu_{1,l})^2]>0)]\nonumber\\
        &=\frac{1}{2}[\mathbb{P}_{\zz_1\sim D_1}(\sum_{l=1}^d\omega_l^2(1-\wz_{1,l})(\wu_{1,l}-\wu_{2,l})>0)+\mathbb{P}_{\zz_2\sim D_2}(\sum_{l=1}^d\omega_l^2(1-\wz_{2,l})(\wu_{2,l}-\wu_{1,l})>0)]\nonumber\\
        &\leq\frac{2\sum_{l=1}^d\omega_l^4(\wu_{1,l}-\wu_{2,l})^2(\ws_{1,l}^2+\ws_{2,l}^2)}{(\sum_{l=1}^d\omega_l^2(\wu_{1,l}-\wu_{2,l})^2)^2} \quad\quad\quad \text{[Applying Lemma \ref{cantelli}]}\nonumber\\
        &\leq\frac{2\sum_{l=1}^d\omega_l^4(\wu_{1,l}+\wu_{2,l})^2(\ws_{1,l}^2+\ws_{2,l}^2)}{(\sum_{l=1}^d\omega_l^2(\wu_{1,l}-\wu_{2,l})^2)^2}\quad\quad\quad \text{[} \wu_{1,l}\wu_{2,l}\geq0\text{]}\nonumber\\
        &=\frac{8\sum_{l=1}^d\omega_l^4(\ws_{1,l}^2+\ws_{2,l}^2)}{(\sum_{l=1}^d\omega_l^2(\wu_{1,l}-\wu_{2,l})^2)^2}\quad\quad\quad\quad\quad\quad\quad\ \,\text{[Standadization: } (\wu_{1,l}+\wu_{2,l})/2=1 \text{]}\nonumber\\
        &\leq\frac{8\sum_{l=1}^d\omega_l^4(\ws_{1,l}+\ws_{2,l})^2}{(\sum_{l=1}^d\omega_l^2(\wu_{1,l}-\wu_{2,l})^2)^2}.\quad\quad\quad\quad\quad\quad\quad\,\text{[} \ws_{1,l}\ws_{2,l}\geq0 \text{]}
        % &=\frac{2\sum_{l=1}^N\omega_l^4(\mu_{1,l}-\mu_{2,l})^2(\sigma_{1,l}^2+\sigma_{2,l}^2)}{(\sum_{l=1}^N\omega_l^2(\mu_{1,l}-\mu_{2,l})^2)^2}\\
        % &<\frac{2\sum_{l=1}^N\omega_l^4(\mu_{1,l}-\mu_{2,l})^2(\sigma_{1,l}^2+\sigma_{2,l}^2)}{\sum_{l=1}^N\omega_l^3(\mu_{1,l}-\mu_{2,l})^4}\\
        % &<\sum_{l=1}^N\frac{2\omega_l(\sigma_{1,l}^2+\sigma_{2,l}^2)}{(\mu_{1,l}-\mu_{2,l})^2}\\
        % &<\sum_{l=1}^N\frac{2\omega_l(\sigma_{1,l}+\sigma_{2,l})^2}{(\mu_{1,l}-\mu_{2,l})^2}
    \end{align}

Let $x_l=\omega_l^2, a_l=(\wu_{1,l}-\wu_{2,l})^2, b_l=(\ws_{1,l}+\ws_{2,l})^2$, then according to Lemma \ref{minimum}, the minimum value of the upper bound (19) is reached when $\omega_l^2\propto\frac{(\wu_{1,l}-\wu_{2,l})^2}{(\ws_{1,l}+\ws_{2,l})^2}=\frac{(\mu_{1,l}-\mu_{2,l})^2}{(\sigma_{1,l}+\sigma_{2,l})^2}$, i.e., $\omega_l\propto\frac{|\mu_{1,l}-\mu_{2,l}|}{\sigma_{1,l}+\sigma_{2,l}}$.
\end{proof}

% \section{Properties of the Simple Transformation}
% When $x>0$, $\phi(x) = \frac{1}{ln^k(\frac{1}{x}+1)}$. Then $\phi'(x)=\frac{k}{(x^2+x)ln^{k+1}(\frac{1}{x}+1)}>0$. Since $\lim_{x\to\infty}\frac{x}{lnx}=+\infty$, we have 
% \begin{equation}
% \begin{split}
%     &\lim_{x\to 0^+}\phi'(x)\\
%     =&\lim_{x\to 0^+}\frac{k}{(x^2+x)ln^{k+1}(\frac{1}{x}+1)}
% \end{split}
% \end{equation}

\section{Training and Evaluation Details}
\label{secdetails}
For S2M2 and MoCo-v2 in Table \ref{performance}, we directly use the official publicly-available pre-trained checkpoints. All other algorithms in Table \ref{performance} are trained using a learning rate 0.1 with cosine decay schedule without restart. SGD with momentum 0.9 is adopted as the optimizer. For all meta-learning algorithms, a total of 60000 5-way 5-shot tasks are sampled for training, each of which contains 15 query images per class. The batch size (number of sampled tasks of each iteration) is 4. All other hyperparameters of MetaOpt match the default settings in the original paper. All conventionally-trained algorithms are trained for 60 epochs, and the batch size is set to 128. For the training of the CE (Cross-Entropy) algorithm, we normalize the representation before the fully-connected layer. We find that if we do not  normalize the representation during the training of CE, the simple transformation does not work. We leave it for future work to investigate this phenomenon.

For the test-time linear classification method we implement for MoCo and S2M2 in Table \ref{performance} and Figure \ref{shot_analysis}, we adopt the Logistic Regression  implementation of scikit-learn~\cite{scikit}.

% the total steps of finetuning needs to be controlled

\section{The Effect of Hyperparameter $k$}
\label{seceffectk}
In Figure \ref{k_experiment}, we show how the hyperparameter $k$ in Eq. (\ref{simple_transformation}) influences the few-shot classification performance. On all datasets, As the $k$ becomes larger, the accuracy first increases and then decreases. The optimal value of $k$ varies for different datasets, ranging from $0.6$ to $1.8$. That being said, the simple transformation gives a relatively stable performance improvement on all datasets when $k\in[1,2]$. Notably, datasets with larger task distribution shift often give a smaller optimal $k$. This phenomenon is reasonable because as seen from Figure \ref{transformation_plot}, a smaller $k$ leads to a larger smoothing effect, and the transformation can better rectify the channel distribution when the task distribution shift is also larger. 

% different datasets have differnt optimal value of $k$

\begin{figure*}[t]
% \vskip 0.2in

\centering
\centerline{\includegraphics[width=1.0\linewidth]{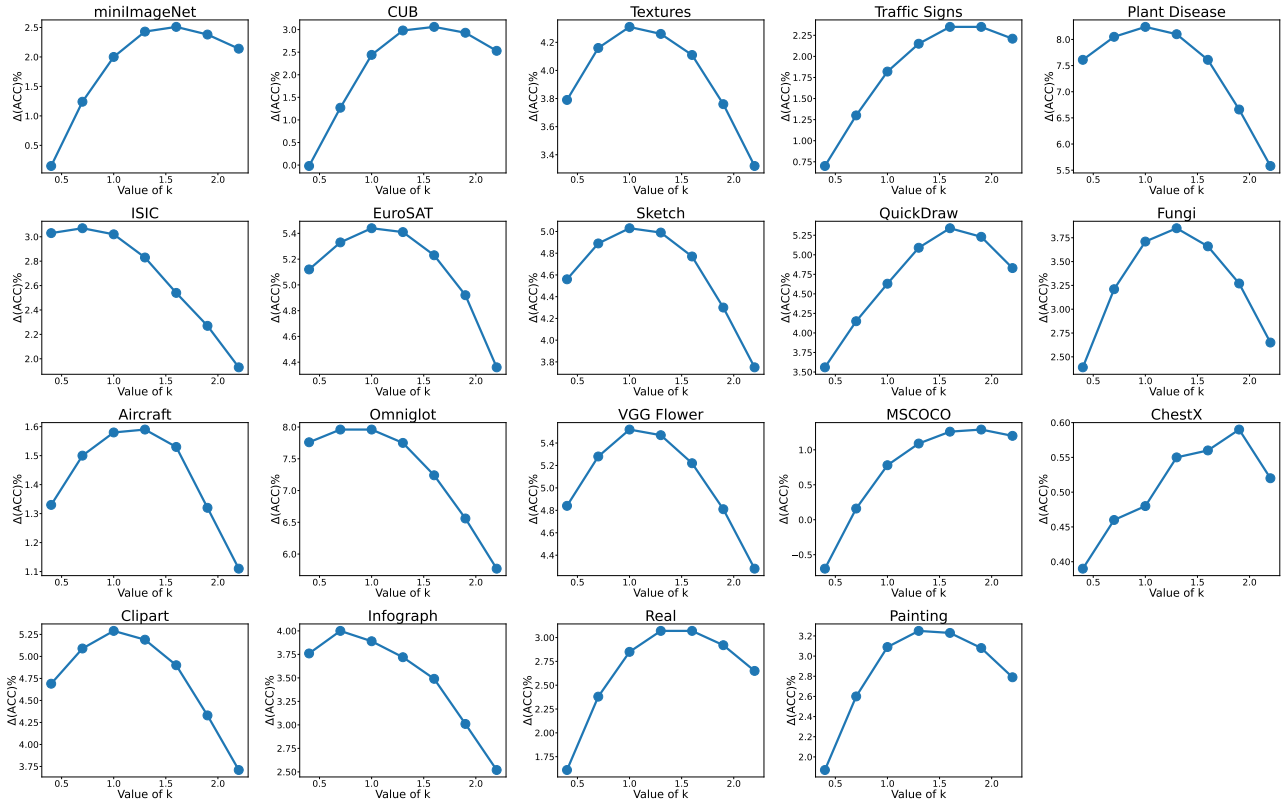}}
% \hfill
\caption{The effect of the hyperparameter $k$ in the simple transformation on each test-time dataset. The feature extractor is trained on \emph{mini}ImageNet using PN. All experiments are conducted on 10000 5-way 5-shot tasks sampled with a fixed seed.}
\label{k_experiment}
\vskip -0.2in
\end{figure*}

\section{Attempts at Handling Negative Output Values of Neural Networks}
\label{secattempt}
As shown in Section \ref{sec3}, the MMC of image representations can represent the emphasis of neural network on different channels. However, things get complicated if the output of the neural network can take negative values. If the value of a channel represents the activation of a feature, it is difficult to say whether a large negative value means a large or small activation. At this time, we consider negative value as a signal of feature activation as well, which leads to the following simple extension of the transformation:

% It is found in experiments that a large-absolute-value negative number seems to represent a large activation, which means that the channel is important. Therefore, we consider the following extension of the simple transformation:

\begin{equation}
    \phi_k(\lambda) =\begin{cases} \frac{\mathrm{sign}(\lambda)}{ln^k(\frac{1}{|\lambda|}+1)}, &|\lambda|>0\\
    0, &\lambda=0
    \end{cases}
\end{equation}

now this extended simple transformation can be applied directly to the standard ResNet-12 using leaky ReLU, and the results are shown in Table \ref{positive}. While leaky ReLU improves the basic performance compared to vanilla ReLU, the improvement of the simple transformation becomes significantly smaller. We conjecture that a large negative magnitude of a channel does not strictly mean that this channel is important. It is future work to investigate how to exactly measure channel importance in such circumstances.

% We conjecture that the negative values 

\begin{table*}[t]
\setlength\tabcolsep{4pt}
% \footnotesize
% \scriptsize
\caption{Performance gains when applying the extended version of the simple transformation to ResNet-12 with Leaky ReLU trained on mini-train.}
\label{positive}
\centering
\begin{tabular}{c|cccccccccc|c}
% \Cline{0.6pt}{1-8}
\\[-1em]
 Algorithm  & \emph{mini}-test & CUB & Texture & TS & PlantD & ISIC & ESAT & Sketch & QDraw & Fungi & Avg
\\

% \multirow{2}{*}{Algorithm}  &
% \multicolumn{2}{c}{$\mathcal{D}_v$-$\mathrm{Ori}$} & \multicolumn{2}{c}{$\mathcal{D}_v$-$\mathrm{FG}$} \\
% \Cline{0.6pt}{1-8}
PN & \acc{76.0}{0.5} &  \acc{59.3}{0.6} & \acc{62.0}{0.7} & \macc{66.3}{0.2} & \acc{78.2}{2.8} & \acc{38.1}{1.3} & \acc{75.1}{0.8} & \acc{52.7}{0.3} & \acc{66.5}{2.9} & \acc{55.4}{0.0} & \acc{63.0}{1.0}
\\
CE & \acc{79.4}{0.2} & \acc{64.5}{0.8}  & \macc{66.1}{0.2} & \acc{69.9}{0.1} & \acc{84.9}{2.1} & \acc{40.1}{0.1} & \acc{77.6}{0.4} & \acc{53.6}{0.4} & \acc{72.1}{3.7}
& \acc{57.4}{1.1} & \acc{66.6}{0.9}

\end{tabular}
% \vskip -0.5in
% 
\end{table*}

% \section{}

\begin{figure*}[t]
% \vskip 0.2in

\centering
\centerline{\includegraphics[width=1.0\linewidth]{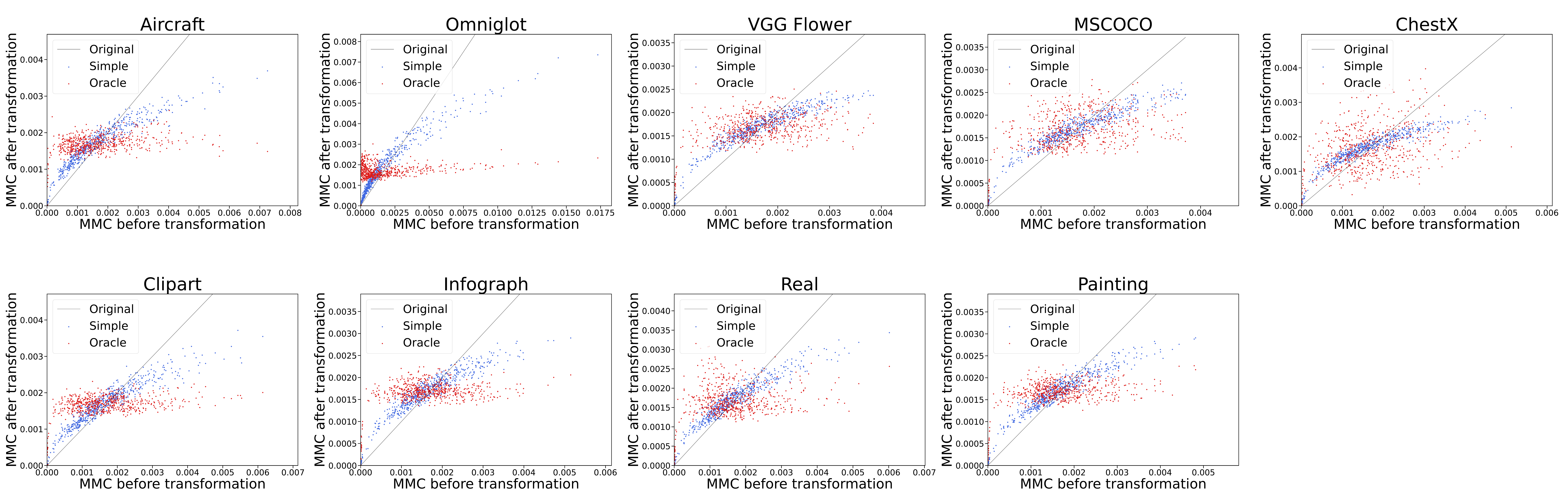}}
% \hfill
\caption{\textbf{The Visualization of MMC of the other nine datasets before and after the use of simple and oracle transformation.} All notations are the same as in Figure \ref{visiualize_optimal}.}

% \textbf{Visualization of MMC of ten datasets $\oo_D$ before and after the use of simple and oracle transformation.} In each plot, a point represents a channel, and the x-axis and y-axis represents the MMC before and after transformation respectively, averaged over all possible binary tasks in the corresponding dataset. For comparison, we also plot the line $y=x$ representing the ``None'' scenario where none of the transformations are applied to features. The feature extractor is trained using PN on \emph{mini}ImageNet.}

\vskip -0.15in
\label{visiualize_optimal_2}
\end{figure*}

\begin{figure*}[t]
% \vskip 0.2in

\centering
\centerline{\includegraphics[width=1.0\linewidth]{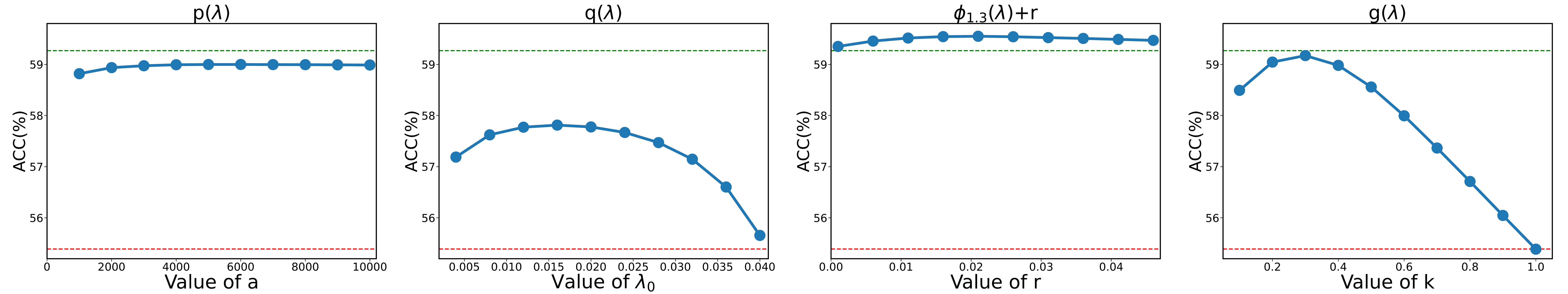}}
% \hfill
\caption{Exploration of necessary ingredients for a good channel-wise transformation. The accuracies show the average 5-way 5-shot performance over all 19 datasets using PN trained on \emph{mini}ImageNet. The red dashed line shows the original performance; the green dashed line shows the performance when using the simple transformation $\phi_{1.3}(\lambda)$. The leftmost plot shows the performance when using the function $p(\lambda)=ln(a\lambda+1)$. The second plot shows the performance when using the piece-wise function $q(\lambda)$. The third plot shows the performance when adding a constant $r$ to the simple transformation $\phi_{1.3}(\lambda)$. The rightmost plot shows the performance when using the power function $g(\lambda)=x^k$.}

\vskip -0.15in
\label{visualize_property}
\end{figure*}

\section{Necessary Ingredients for a Good Transformation}
\label{ingredient}
One may ask that whether all of the three properties presented in Eq. (\ref{property}) are necessary for a transformation to successfully improve few-shot learning performance, or whether there exist good transformations other than $\phi_k(\lambda)$ considered in the main article. To verify the necessity of all properties, we design several functions, each of which does not satisfy one of the properties. First, we consider the function $p(\lambda)=ln(a\lambda+1)$, where $a>0$. This function has positive derivative and negative second derivative, but does not have large enough derivative near zero ($p'(0)=a$). In the left plot of Figure \ref{visualize_property}, we see that the improvement brought by this function is smaller than $\phi_{1.3}(\lambda)$, and that the gap becomes smaller when $a$ 
increases. This validates the necessity of having a large enough derivative near zero. We then consider the piece-wise function 

\begin{equation}
    q(\lambda) =\begin{cases} \phi_k(\lambda), &0\leq \lambda<\lambda_0\\
    a_2\lambda^2+a_1\lambda+a_0, &\lambda\geq \lambda_0
    \end{cases}
\end{equation}

where the values of $a_2,a_1,a_0$ ensure the smoothness of $q(\lambda)$ at $\lambda=\lambda_0$ up to first derivative, and also control the position of the extreme point $x_0=-\frac{a_1}{2a_2}$. This function does not have positive derivative when $\lambda\geq x_0$. We set $x_0=0.05$, and change the value of $\lambda_0$. The results are shown in the second plot in Figure \ref{visualize_property}. As seen, introducing negative derivative into the transformation substantially degrades performance. Finally, the property of having negative second derivative can be naturally broken by increasing the value of $k$ in $\phi_k(\lambda)$, and Figure \ref{k_experiment} shows that doing this would degrade performance. 

Since inactivate channels may represent absence of a feature in a task instead of having low emphasis, thus they are likely to have no importance. Therefore another property $\phi_k(0)=0$, not shown in Eq. (\ref{property}), could also be important for a good transformation. To investigate this, we add a constant $r$ to $\phi_k(\lambda)$ and see how the performance would change. The third plot in Figure \ref{visualize_property} shows the opposite result: a small constant added to the transformation helps further improve the performance. As adding this constant has more influence on small-valued channels, we conjecture that this helps further alleviate the channel bias problem, and that some inactivate channels indeed should gain some focus.

Apparently, $\phi_k(\lambda)$ is not the only function that satisfies all the three properties. We consider the power function $g(\lambda)=\lambda^k$, where $k>0$. Although being very simple, this function matches all desired properties. The rightmost plot in Figure \ref{visualize_property} shows that this function can indeed improve the performance as well. Note that this function has been used in~\cite{free_lunch}, where it is called the Tukey’s Ladder of Powers transformation~\cite{tukey}, and is used to transform the feature distribution to be more like a Gaussian distribution. Here we show that mitigating the channel bias problem may be another reason for why it works.

\section{Details of MMC Calculation and Comparison}
\subsection{Oracle Transformation}
\label{MMC_calculation}
To apply the oracle transformation, for every test-time dataset $D$, we first calculate feature mean $\uu_c$ and variance $\ss_c$ of each class $c$ in $D$. Then for every sampled binary classification task $\tau=\{\mathcal{S}_{\tau}, \mathcal{Q}_{\tau}\}$ that aims at discriminating two classes $c_1$ and $c_2$, we calculate the oracle MMC $\oo$ directly from Eq. (\ref{oracle_adjustment}). Next, we standardize each image feature $\zz$ in $\mathcal{S}_{\tau}$ and $\mathcal{Q}_{\tau}$, and multiply it by $\oo$ to obtain the transformed feature $\zz\leftarrow\oo\odot\hz$. The transformed features can be already used for classification, but we find that for some channels with very small means $\mu_l$, the corresponding value of  oracle MMC $\omega_l$ becomes too big, deviating from what we expect. To avoid generating such outliers, we additionally restrict that for every channel $l$, the ratio of transformed MMC to the original MMC should not surpass a threshold, i.e., $\omega_l/\omega^o_l\leq\alpha$ for some $\alpha\in\mathbb{R}^+$. If a channel $l$ does not meet this requirement, we simply set $\omega_l=\omega^o_l$. In all of our experiments, we set $\alpha=50$. The optimal MMC visualizied in Figure \ref{visiualize_optimal} and Figure \ref{visiualize_optimal_2} is also computed using this strategy. We leave it for future work to investigate the reason behind such phenomenon.
\subsection{Choice of Distance Measure}
\label{distance_clarification}
The normalized mean square difference $d(\xx,\bm{y})=\frac{1}{d}\sum_{l=1}^d(x_l-y_l)^2/x_l^2$ has the advantage of having equal treatment for both channels with small and large values. However, it can be largely influenced by ``outlier channel'' with very small $x_l$. Since dataset-level MMC $\oo_D$ is averaged over MMCs of all possible tasks in $D$, it is more stable and can use such distance measure. The task-level MMC and image features have much higher variances acorss tasks/images, thus for these two fine-grained settings we just use the mean square difference $d(\xx,\bm{y})=\frac{1}{d}\sum_{l=1}^d(x_l-y_l)^2$. This, however, introduces another problem that critical changes of channels with small values are always ignored by such unnormalized distance. Let's see a simple example. Let $\oo_1=(0.05,0.08,0.87)$ and $\oo_2=(0.4,0.3,0.3)$ be two 3-dimensional $l1$-normalized MMCs. Assume that after transformation, their $l1$-normalized values become $\oo_1'=(0.15,0.1,0.75)$ and $\oo_2'=(0.55,0.22,0.23)$. The value of the first dimension of $\oo_1$ triples and surpasses that of the second dimension after transformation, thus the channel emphasis changes substantially, while the channel emphasis of $\oo_2$ does not change much. We expect that the distance measuring the change of $\oo_1$ should be much larger than that measuring the change of $\oo_2$. The normalized mean square differences between the MMC before and after transformation are $1.36$ and $0.09$ for $\oo_1$ and $\oo_2$ respectively, which is in line with our intuition. However, the mean square differences are $0.008$ and $0.011$ for $\oo_1$ and $\oo_2$ respectively. Thus under such circumstances, the normalized  mean square difference is a much better choice. Although being simple, $\oo_2$ and $\oo_1$ are a good analogy to the MMC pattern on \emph{mini}ImageNet and some other datasets in Figure \ref{visiualize_optimal}, respectively. In Figure  \ref{visiualize_optimal}, we can see that most MMC values on \emph{mini}ImageNet are around mid-level, which resembles $\oo_2$; most MMC values on other datasets are either very small or large, which resembles $\oo_1$. This explains why in Table \ref{distance} the task-level and image-level differences on \emph{mini}ImageNet are not smaller than those on other datasets. We leave it for future work to find a distance measure that could avoid unstable results, while being sensitive to small-valued channels.

% \section{The Simple Transformation.}

\section{More Details on Fine-tuning Based Method}
For fine-tuning methods in Figure \ref{shot_analysis}, we grid search the best hyperparameters in each shot setting \emph{on the test set}. All best configurations are shown in Table \ref{hyperpameter}. As seen, the hyperparameters of fine-tuning methods are very sensitive to the number of shots. In low-shot settings, care should be taken for controlling the total steps of finetuning and learning rate, in order to avoid overfitting. This phenomenon is also shown in~\cite{howtotrainmaml}, where the authors show that MAML~\cite{maml}, one of the most widely adopted finetuning-based methods, has a much higher optimal test-time fine-tuning steps than expected.

% This phenomenon is also found in~\cite{howtotrainmaml}, where they

\begin{table*}[t]
% \setlength\tabcolsep{4pt}
% \footnotesize
% \scriptsize
\caption{Found best hyperparameters of the test-time finetuning method on \emph{mini}ImageNet. The feature extrator is ResNet-12, trained by S2M2 algorithm.}
\label{hyperpameter}
\centering
\begin{tabular}{c|ccc}
% \hline
Shot & Batch size & Number of epochs & Learning rate\\\hline
1 & 5 & 10 & 0.1\\
5 & 25 & 30 & 0.05\\
10 & 50 & 50 & 0.05\\
20 & 50 & 100 & 0.02\\
50 & 64 & 100 & 0.01\\
100 & 64 & 500 & 0.005\\
400 & 64 & 500 & 0.005\\

\end{tabular}
% \vskip -0.5in
% 
\end{table*}

\section{Error Bars}
All results regarding performance in the main paper are shown without error bars. In Table \ref{seed_impact}, we show how different seeds affect the improvement brought by the simple transformation $\phi_k(\lambda)$. There are two seeds that could influence the result, one for training, and one for testing. When considering test seed, we fix the feature extractor and use different seeds to sample tasks; when considering train seed, we fix the test seed (same tasks) and evaluate different feature extractors trained with different seeds. As seen, while varying the test seed hardly affect the performance, varying the train seed produces some fluctuations. 
After considering the fluctuations, the improvement given by the transformation can still be statistically guaranteed.

\begin{table*}[t]
\setlength\tabcolsep{2.3pt}
% \footnotesize
% \scriptsize

\caption{The influence of using different seeds during training or testing. The feature extractor is trained by PN on \emph{mini}-train. 
Average 5-way 5-shot performance gains brought by the simple transformation on 10 datasets with 95\% confidence interval (over 5 trials) are shown.}
\centering
\begin{tabular}{c|cccccccccc}
\label{seed_impact}
% \Cline{0.6pt}{1-8}
\\[-1em]
 Seed  & \emph{mini}-test & CUB & Texture & TS & PlantD & ISIC & ESAT & Sketch & QDraw & Fungi
\\

% \multirow{2}{*}{Algorithm}  &
% \multicolumn{2}{c}{$\mathcal{D}_v$-$\mathrm{Ori}$} & \multicolumn{2}{c}{$\mathcal{D}_v$-$\mathrm{FG}$} \\
% \Cline{0.6pt}{1-8}
Test & \pacc{+2.36}{0.06}  &\pacc{+2.98}{0.03} &\pacc{+4.26}{0.10} &\pacc{+2.37}{0.06} &\pacc{+8.04}{0.14} &\pacc{+2.80}{0.10} &\pacc{+5.50}{0.09} &\pacc{+5.02}{0.09} &\pacc{+5.46}{0.14} &\pacc{+3.90}{0.11}

% & 2.98 \pm 0.03 & 4.26 \pm 0.10 & 2.37 \pm 0.06 & 8.04 \pm 0.14 & 2.80 \pm 0.10 & 5.50 \pm 0.09 & 5.02 \pm 0.09 & 5.46 \pm 0.14 &   3.90 \pm 0.11
% \acc{62.0}{0.7} & \macc{66.3}{0.2} & \acc{78.2}{2.8} & \acc{38.1}{1.3} & \acc{75.1}{0.8} & \acc{52.7}{0.3} & \acc{66.5}{2.9} & \acc{55.4}{0.0} & \acc{63.0}{1.0}
\\
Train & \pacc{+2.08}{0.33} & \pacc{+3.28}{0.43} & \pacc{+3.78}{0.77} & \pacc{+2.00}{0.89} & \pacc{+8.21}{0.80} & \pacc{+3.64}{0.44} & \pacc{+4.03}{1.37} & \pacc{+4.21}{0.93} & \pacc{+7.01}{2.62} & \pacc{+3.59}{0.50}
% CE & \acc{79.4}{0.2} & \acc{64.5}{0.8}  & \macc{66.1}{0.2} & \acc{69.9}{0.1} & \acc{84.9}{2.1} & \acc{40.1}{0.1} & \acc{77.6}{0.4} & \acc{53.6}{0.4} & \acc{72.1}{3.7}
% & \acc{57.4}{1.1} & \acc{66.6}{0.9}

\end{tabular}
% \vskip -0.5in
% 
\end{table*}

\end{document}